\documentclass{article} 
\usepackage{iclr2026_conference,times}

\usepackage{custom}

\newcommand{\mymacro}[1]{{#1}}

\newcommand{\justification}[1]{%
    \refstepcounter{equation}%
    \tag{\theequation \textcolor{black!50}{, \footnotesize{#1}}}
}

\newcommand{\ignore}[1]{}
\newcommand{\expandLater}[1]{}

\newcommand{\defn}[1]{\textbf{#1}}

\newcommand{\ifcondition}{\textbf{if }}
\newcommand{\otherwisecondition}{\textbf{otherwise }}

\newcommand{\defeq}{\mathrel{\stackrel{\textnormal{\tiny def}}{=}}}
\newcommand{\inv}[1]{{#1^{\mymacro{-1}}}}

\newcommand{\ind}[1]{\mymacro{\mathbbm{1}} \left\{ #1 \right\}}
\DeclarePairedDelimiter\ceil{\lceil}{\rceil}
\newcommand{\floor}[1]{\mymacro{\left\lfloor #1 \right\rfloor}}
\newcommand{\posPart}[1]{\mymacro{\left[ #1 \right]_{+}}}
\newcommand{\set}[1]{{\left\{ #1 \right\}}}

\def\1{\mymacro{\bm{1}}}
\def\0{\mymacro{\bm{0}}}

\newcommand{\interleaveFun}[2]{\mymacro{{#1}^\frown{#2}}}

\newcommand{\N}{{\mymacro{\mathbb{N}}}}

\newcommand{\NTo}[1]{{\left[ #1 \right]}}

\newcommand{\R}{{\mymacro{\mathbb{R}}}}
\newcommand{\F}{{\mymacro{\mathbb{F}}}}
\newcommand{\maxFNum}{{\mymacro{B_{\F}}}}

\newcommand{\abs}[1]{{\mymacro{\left\lvert #1 \right\rvert}}}

\newcommand{\bigO}{{\mymacro{\mathcal{O}}}}
\newcommand{\bigOFun}[1]{{\mymacro{\bigO(#1)}}}
\newcommand{\bigOLog}{{\mymacro{\widetilde{\bigO}}}}
\newcommand{\bigOLogFun}[1]{{\mymacro{\bigOLog(#1)}}}
\newcommand{\bigOmega}{{\mymacro{\Omega}}}
\newcommand{\bigOmegaFun}[1]{{\mymacro{\bigOmega(#1)}}}
\newcommand{\bigTheta}{{\mymacro{\Theta}}}
\newcommand{\bigThetaFun}[1]{{\mymacro{\bigTheta(#1)}}}

\DeclareMathOperator*{\argmax}{{\mymacro{argmax}}}

\DeclareMathSymbol{\mlq}{\mathord}{operators}{``} 
\DeclareMathSymbol{\mrq}{\mathord}{operators}{`'} 

\newcommand{\pdens}{{\mymacro{p}}}
\newcommand{\qdens}{{\mymacro{q}}}

\newcommand{\probOver}{{\mymacro{\mathcal{P}}}}
\newcommand{\probOverFun}[1]{{\mymacro{\probOver\mleft(#1\mright)}}}

\newcommand{\inner}[2]{{\mymacro{#1^\top #2}}}

\newcommand{\alphabet}{{\mymacro{\Sigma}}}
\newcommand{\strings}{{\mymacro{\kleene{\alphabet}}}}
\newcommand{\nsymbols}{{\mymacro{|\alphabet|}}}
\newcommand{\eosnsymbols}{{\mymacro{|\eosalphabet|}}}

\newcommand{\eosalphabet}{{\mymacro{\overline{\alphabet}}}}
\newcommand{\recognizer}{\mymacro{R}}
\newcommand{\recognizerFun}[1]{\mymacro{\recognizer}\mleft(#1\mright)}
\newcommand{\lang}{\mymacro{\mathcal{L}}}
\newcommand{\langFun}[1]{\mymacro{\lang}\mleft(#1\mright)}
\newcommand{\kleene}[1]{{\mymacro{#1^{*}}}}

\newcommand{\str}{{\mymacro{\boldsymbol{w}}}}
\newcommand{\unmstr}{{\mymacro{\boldsymbol{y}}}}
\newcommand{\unmsym}{{\mymacro{y}}}
\newcommand{\strlen}{{\mymacro{N}}}

\newcommand{\stridx}{{\mymacro{n}}}

\newcommand{\stru}{{\mymacro{\boldsymbol{u}}}}

\newcommand{\sym}{{\mymacro{w}}}
\newcommand{\eossym}{{\mymacro{\overline{\sym}}}}

\newcommand{\diffusionProcess}{\mymacro{q}}
\newcommand{\learnedGenerativeProcess}{\mymacro{\widehat{q}}}
\newcommand{\maskSchedule}{\mymacro{\alpha}}
\newcommand{\maskScheduleFun}[1]{\mymacro{\maskSchedule\mleft(#1\mright)}}
\newcommand{\maskSym}{\mymacro{\texttt{m}}}
\newcommand{\thinkstr}{\mymacro{\boldsymbol{t}}}

\newcommand{\acceptSym}{\mymacro{\texttt{1}}}
\newcommand{\rejectSym}{\mymacro{\texttt{0}}}
\newcommand{\padlen}{\mymacro{P}}
\newcommand{\padSym}{\mymacro{\square}}
\newcommand{\timesteps}{\mymacro{T}}
\newcommand{\maskAlphabet}{\mymacro{\alphabet_{\maskSym}}}

\newcommand{\acceptAlphabet}{\mymacro{\alphabet_{01}}}
\newcommand{\maskAcceptAlphabet}{\mymacro{\alphabet_{\maskSym 01}}}
\newcommand{\planner}{\mymacro{U}}
\newcommand{\predictor}{\mymacro{S}}
\newcommand{\mdmModel}{\mymacro{M}}
\newcommand{\modelClass}{\mymacro{\mathtt{C}}\xspace}

\newcommand{\LTAcr}{\mymacro{PLT}\xspace}
\newcommand{\LTsAcr}{\mymacro{PLTs}\xspace}
\newcommand{\MDMAcr}{\mymacro{MDM}\xspace}
\newcommand{\MDMsAcr}{\mymacro{MDMs}\xspace}
\newcommand{\eMDMAcr}{\mymacro{MDM}\xspace}
\newcommand{\eMDMsAcr}{\mymacro{MDMs}\xspace}
\newcommand{\sMDMAcr}{\mymacro{sMDM}\xspace}
\newcommand{\sMDMsAcr}{\mymacro{sMDMs}\xspace}
\newcommand{\CoTAcr}{\mymacro{CoT}\xspace}

\newcommand{\pCoTAcr}{\mymacro{pCoT}\xspace}

\newcommand{\LTClass}{\mymacro{\mathtt{PLT}}\xspace}

\newcommand{\CoTClass}{\mymacro{\mathtt{CoT}}\xspace}

\newcommand{\pCoTClass}{\mymacro{\mathtt{pCoT}}\xspace}

\newcommand{\MDMClass}{\mymacro{\mathtt{MDM}}\xspace}
\newcommand{\MDMClassEdit}{\mymacro{\mathtt{MDM}}\xspace}
\newcommand{\MDMClassSimple}{\mymacro{\mathtt{sMDM}}\xspace}

\newcommand{\poly}{\mymacro{\mathtt{poly}}}
\newcommand{\polyFun}[1]{\mymacro{\poly\mleft(#1\mright)}}

\newcommand{\bin}{\mymacro{\texttt{B}}}
\newcommand{\binFun}[1]{\mymacro{\bin\mleft(#1\mright)}}
\newcommand{\sbin}{\mymacro{\bin^{\pm}}}
\newcommand{\sbinFun}[1]{\mymacro{\sbin\mleft(#1\mright)}}
\newcommand{\symbolembedding}{\mymacro{\bm{e}}}
\newcommand{\posencoding}{\mymacro{\bm{p}}}
\newcommand{\transoutput}{\mymacro{\bm{o}}}

\newcommand{\circuit}{\mymacro{C}}
\newcommand{\circuitFun}[1]{\mymacro{\circuit\mleft(#1\mright)}}
\newcommand{\circuitFamily}{\mymacro{\sC}}

\newcommand{\LUniform}{\mymacro{\LCls\textnormal{-uniform}}}
\newcommand{\LUniformity}{\mymacro{\LCls\textnormal{-uniformity}}}

\newcommand{\XUniform}{\mymacro{\XCls\text{-uniform}}}

\newcommand{\LCls}{\mymacro{\mathtt{L}}}

\newcommand{\XCls}{\mymacro{\mathtt{X}}}

\newcommand{\andGate}{\mymacro{\texttt{AND}}\xspace}
\newcommand{\orGate}{\mymacro{\texttt{OR}}\xspace}
\newcommand{\notGate}{\mymacro{\texttt{NOT}}\xspace}

\newcommand{\round}{\mymacro{\texttt{round}}}
\newcommand{\roundFun}[1]{\round\mleft(#1\mright)}
\newcommand{\roundOpFun}[1]{\mleft[#1\mright]_\numPrec}

\newcommand{\polyCls}{{\mymacro{\textnormal{\textsf{\small P}}}}}
\newcommand{\AC}{{\mymacro{\mathtt{AC}}}}
\newcommand{\ACFun}[1]{{\mymacro{\AC^{\mathtt{#1}}}}}
\newcommand{\ACZero}{{\mymacro{\ACFun{0}}}}

\newcommand{\ACd}{{\mymacro{\ACFun{d}}}}
\newcommand{\TC}{{\mymacro{\mathtt{TC}}}}
\newcommand{\TCFun}[1]{{\mymacro{\TC^{\mathtt{#1}}}}}
\newcommand{\TCZero}{{\mymacro{\TCFun{0}}}}

\newcommand{\TCd}{{\mymacro{\TCFun{d}}}}
\newcommand{\NC}{{\mymacro{\mathtt{NC}}}}
\newcommand{\NCFun}[1]{{\mymacro{\NC^{\mathtt{#1}}}}}

\newcommand{\NCOne}{{\mymacro{\NCFun{1}}}}

\newcommand{\langEnc}{\mymacro{{\textnormal{\texttt{Enc}}}}}
\newcommand{\langEncFun}[1]{\mymacro{\langEnc\mleft(#1\mright)}}

\newcommand{\pLM}{\mymacro{\overset{\curvearrowright}{\pdens}}}
\newcommand{\pLMFun}[1]{\mymacro{\pLM(#1)}}
\newcommand{\pLMInf}{\mymacro{\pdens^{\downarrow}}}
\newcommand{\pLMInfFun}[1]{\mymacro{\pLMInf(#1)}}
\newcommand{\qLM}{\mymacro{\qdens}}

\newcommand{\eos}{{\mymacro{\textsc{eos}}}}

\newcommand{\ngram}{{\mymacro{\textit{n}-gram}}\xspace}

\newcommand{\binEnc}[1]{{\mymacro{\bin\mleft(#1\mright)}}}
\newcommand{\fullBinEnc}[1]{{\mymacro{\overline{\bin}\mleft(#1\mright)}}}

\newcommand{\onehot}[1]{{\mymacro{\llbracket#1\rrbracket}}}

\newcommand{\one}{{\mymacro{\bm{1}}}}

\newcommand{\tm}{{\mymacro{\mathcal{M}}}}

\newcommand{\inMtx}{{\mymacro{\mV}}}
\newcommand{\outMtx}{{\mymacro{\mE}}}

\newcommand{\hiddDim}{{\mymacro{D}}}
\newcommand{\hiddDimFun}[1]{{\mymacro{\hiddDim\mleft(#1\mright)}}}
\newcommand{\dimidx}{{\mymacro{d}}}

\newcommand{\numPrec}{{\mymacro{\texttt{p}}}}
\newcommand{\numPrecFun}[1]{{\mymacro{\numPrec\mleft(#1\mright)}}}

\newcommand{\softmax}{{\mymacro{\mathrm{softmax}}}}

\newcommand{\ReLU}{{\mymacro{\mathrm{ReLU}}}}
\newcommand{\softmaxFun}[2]{{\mymacro{\mathrm{softmax}\mleft(#1\mright)_{#2}}}} 
\newcommand{\ReLUFun}[1]{{\mymacro{\ReLU\mleft(#1\mright)}}} 

\newcommand{\hiddState}{{\mymacro{\vh}}}
\newcommand{\ehiddState}{{\mymacro{\evh}}}

\newcommand{\grammar}{{\mymacro{\mathcal{G}}}}

\newcommand{\mlp}{{\mymacro{{f}}}}

\newcommand{\topk}{{\mymacro{\textnormal{{\small \textsf{top-}}}k}}}
\newcommand{\topkFun}[1]{{\mymacro{\topk\mleft(#1\mright)}}}

\newcommand{\negterm}[1]{{\mymacro{{\raise.17ex\hbox{$\scriptstyle\sim$}} #1}}}

\newcommand{\hiddMtx}{{\mymacro{\mH}}}
\newcommand{\queryMtx}{{\mymacro{\mQ}}}
\newcommand{\keyMtx}{{\mymacro{\mK}}}
\newcommand{\valueMtx}{{\mymacro{\mV}}}

\newcommand{\attnMask}{{\mymacro{M}}}
\newcommand{\attnMaskFun}[1]{{\mymacro{\attnMask(#1)}}}

\newcommand{\tf}{{\mymacro{\mathcal{T}}}}

\newcommand{\tflayer}{{\mymacro{\boldsymbol{\tau}}}}
\newcommand{\dumpLayer}{{\mymacro{\tflayer_{\texttt{dump}}}}}
\newcommand{\dumpLayerFun}[1]{{\mymacro{\dumpLayer\mleft(#1\mright)}}}
\newcommand{\readLayer}{{\mymacro{\tflayer_{\texttt{read}}}}}
\newcommand{\readLayerFun}[1]{{\mymacro{\readLayer\mleft(#1\mright)}}}
\newcommand{\decodeStep}{{\mymacro{\texttt{Dec}}}}
\newcommand{\decodeStepFun}[1]{{\mymacro{\decodeStep\mleft(#1\mright)}}}
\newcommand{\decoder}{{\mymacro{\texttt{NS}}}}
\newcommand{\decoderFun}[1]{{\mymacro{\decoder\mleft(#1\mright)}}}
\newcommand{\decoderpar}{{\mymacro{\texttt{NS}_{\|}}}}

\newcommand{\residStream}{{\mymacro{\hiddMtx}}}

\newcommand{\layeridx}{{\mymacro{l}}}
\newcommand{\numlayers}{{\mymacro{L}}}

\newcommand{\paddingBlock}{{\mymacro{b}}}
\newcommand{\relativePos}{{\mymacro{r}}}

\newcommand{\posEnc}{{\mymacro{\texttt{PE}}}}
\newcommand{\posEncFun}[1]{{\mymacro{\posEnc\mleft(#1\mright)}}}

\def\eps{{\mymacro{\varepsilon}}}

\def\vb{{{\mymacro{\bm{b}}}}}

\def\vd{{{\mymacro{\bm{d}}}}}
\def\ve{{{\mymacro{\bm{e}}}}}

\def\vh{{{\mymacro{\bm{h}}}}}

\def\vk{{{\mymacro{\bm{k}}}}}
\def\vl{{{\mymacro{\bm{l}}}}}

\def\vq{{{\mymacro{\bm{q}}}}}

\def\vv{{{\mymacro{\bm{v}}}}}

\def\vx{{{\mymacro{\bm{x}}}}}
\def\vy{{{\mymacro{\bm{y}}}}}

\def\evh{{{\mymacro{h}}}}

\def\evp{{{\mymacro{p}}}}

\def\evs{{{\mymacro{s}}}}

\def\evx{{{\mymacro{x}}}}
\def\evy{{{\mymacro{y}}}}

\def\mA{{{\mymacro{\bm{A}}}}}
\def\mB{{{\mymacro{\bm{B}}}}}

\def\mE{{{\mymacro{\bm{E}}}}}

\def\mG{{{\mymacro{\bm{G}}}}}
\def\mH{{{\mymacro{\bm{H}}}}}

\def\mK{{{\mymacro{\bm{K}}}}}

\def\mQ{{{\mymacro{\bm{Q}}}}}

\def\mV{{{\mymacro{\bm{V}}}}}
\def\mW{{{\mymacro{\bm{W}}}}}

\DeclareMathAlphabet{\mathsfit}{\encodingdefault}{\sfdefault}{m}{sl}
\SetMathAlphabet{\mathsfit}{bold}{\encodingdefault}{\sfdefault}{bx}{n}

\def\sC{{{\mymacro{\mathcal{C}}}}}

\def\sN{{{\mymacro{\mathcal{N}}}}}

\def\sX{{{\mymacro{\mathcal{X}}}}}
\def\sY{{{\mymacro{\mathcal{Y}}}}}

\title{On the Reasoning Abilities of Masked Diffusion Language Models}

\author{Anej Svete\thanks{~~This research was conducted while interning at the Allen Institute for AI.} \\
ETH Zürich \\
\href{mailto:asvete@inf.ethz.ch}{\texttt{asvete@inf.ethz.ch}} \\
\And
Ashish Sabharwal \\
Allen Institute for AI \\
\href{mailto:ashishs@allenai.org}{\texttt{ashishs@allenai.org}} \\
}

\iclrfinalcopy 
\begin{document}

\maketitle

\begin{abstract}
   Masked diffusion models (\MDMsAcr) for text offer a compelling alternative to traditional autoregressive language models. 
   Parallel generation makes them efficient, but their computational capabilities and the limitations inherent in their parallelism remain largely unexplored. 
   To this end, we characterize what types of reasoning problems \MDMsAcr can provably solve and how efficiently. 
   We do this by connecting \MDMsAcr to the well-understood reasoning frameworks of chain of thought (\CoTAcr) and padded looped transformers (\LTsAcr) in the finite-precision log-width setting: We show that \MDMsAcr and polynomially-padded \LTsAcr are, in fact, equivalent in this setting, and that \MDMsAcr can solve all problems that \CoTAcr-augmented transformers can.
   Moreover, we showcase classes of problems (including regular languages) for which \MDMsAcr are inherently more efficient than \CoTAcr transformers, where parallel generation allows for substantially faster reasoning.
\end{abstract}

\section{Introduction} \label{sec:introduction}

Many complex problems can be decomposed into smaller, independent sub-problems, making them naturally suited for parallel computation. 
For example, we can compute the value of a mathematical expression by evaluating its sub-expressions independently and combining the results (see \cref{fig:expression-tree}). 
However, dominant autoregressive language models (LMs) tackle these problems sequentially. 
Methods like chain of thought (\CoTAcr), for instance, generate solutions one step at a time, failing to capitalize on the underlying parallel structure.
Parallel generation by masked diffusion models (\MDMsAcr) offers a compelling alternative.
Recent advances have positioned \MDMsAcr as a viable contender to autoregressive LMs in language modeling, code generation, and even molecule design \citep{Lou2024DiscreteDM,zhang2025surveyparalleltextgeneration,sunSpeedAlwaysWins2025}.
However, the fundamental reasoning capabilities of \MDMsAcr remain poorly understood,
which limits the extent to which we can
leverage their potential and apply them to appropriate tasks.
This work bridges that gap by providing the first formal characterization of the expressivity of \MDMsAcr, clarifying their fundamental computational strengths and weaknesses.

\begin{figure}[ht]
   \centering

    \begin{tikzpicture}[
        node distance=1.5cm and 3cm,
        model/.style={
            rectangle,
            rounded corners=3pt,
            fill=#1!20,
            text width=2.5cm,
            minimum height=1cm,
            align=center,
            font=\small
        },
        arrow_eq/.style={<->, thick, #1},
        arrow_incl_right/.style={right hook->, thick, #1},
        arrow_incl_left/.style={left hook->, thick, #1},
        arrow_sep_left/.style={left hook->, dashed, thick, #1},
        label_style/.style={midway, font=\tiny, align=center, text=black!80},
        source_label_style/.style={midway, font=\tiny, align=center, text=black!80, fill=blue!10, inner sep=2pt, rounded corners=2pt},
        target_label_style/.style={midway, font=\tiny, align=center, text=black!80, fill=orange!10, inner sep=2pt, rounded corners=2pt},
        resource_style/.style={font=\tiny\itshape, text=ETHPetrol!80!black, align=center, midway}
    ]

    \node[model=ETHGreen] (mdm) at (0,0) {Masked Diffusion \\ Models};
    \node[model=ETHBlue] (lt) at (-4.5,0) {Padded Looped \\ Transformers};
    \node[model=ETHPurple] (cot) at (4.5,0) {Chain of Thought};


    \draw[arrow_incl_left=ETHRed] (mdm.175) to[bend right=30]
        node[target_label_style, midway, above] {Added stochasticity}
        node[label_style, midway, below, resource_style] {\cref{thm:lts-mdms}}
        (lt.5);

    \draw[arrow_incl_left=ETHRed] (lt.355) to[bend right=30]
        node[target_label_style, midway, below] {Extra output space}
        node[label_style, midway, above, resource_style] {\cref{thm:lts-mdms}}
        (mdm.185);

    \draw[arrow_incl_right=ETHRed] (mdm.5) to[bend left=30]
        node[target_label_style, midway, above] {Extra steps}
        node[label_style, midway, below, resource_style] {\cref{thm:cot-can-simulate-mdms}}
        (cot.175);
    \draw[arrow_incl_right=ETHRed] (cot.185) to[bend left=30]
        node[target_label_style, midway, below] {Extra output space}
        node[label_style, midway, above, resource_style] {\cref{thm:mdms-can-simulate-cot}}
        (mdm.355);

    \draw[arrow_sep_left=ETHRed] (cot.north) to[bend right=30]
        node[source_label_style, pos=0.5, above, align=center] {\textbf{Sequentiality Bottleneck} under $\log\strlen$ steps.}
        node[label_style, midway, below, resource_style] {\cref{thm:cot-mdm-separation}}
        (mdm.north);

    \end{tikzpicture}

   \caption{
      A summary of masked diffusion model expressivity in relation to padded looped transformers and chain of thought.
      $\sX \hookrightarrow \sY$ indicates the inclusion of $\sX$ in $\sY$.
      Dashed arrows represent strict inclusions.
      \textcolor{ETHRed}{Red arrows} denote novel results.
      Edge labels with \colorbox{blue!10}{blue background} indicate constraints on the \emph{source} node, while labels with \colorbox{orange!10}{orange background} indicate constraints on the \emph{target} node.
      }
   \label{fig:figure-1a}
\end{figure}

\begin{figure}[ht]
   \centering

    \begin{tikzpicture}[
        node distance=1.5cm and 2.5cm,
        model/.style={
            rectangle,
            rounded corners=3pt,
            fill=#1!20,
            text width=2.5cm,
            minimum height=1cm,
            align=center,
            font=\small
        },
        complexity/.style={
            rectangle,
            rounded corners=3pt,
            fill=gray!15,
            minimum width=1cm,
            align=center
        },
        arrow_eq/.style={<->, thick, #1},
        arrow_incl_right/.style={right hook->, thick, #1},
        arrow_sep_right/.style={right hook->, dashed, thick, #1},
        arrow_sep_left/.style={left hook->, dashed, thick, #1},
        label_style/.style={midway, font=\tiny, align=center, text=black!80},
        source_label_style/.style={midway, font=\tiny, align=center, text=black!80, fill=blue!10, inner sep=2pt, rounded corners=2pt},
        target_label_style/.style={midway, font=\tiny, align=center, text=black!80, fill=orange!10, inner sep=2pt, rounded corners=2pt},
        resource_style/.style={font=\tiny\itshape, text=ETHPetrol!80!black, align=center, midway}
    ]

    \node[model=ETHGreen] (mdm) at (0,0) {Masked Diffusion \\ Models};
    \node[model=ETHPurple] (cot) at (5.25,0) {Chain of Thought};

    \node[complexity] (reg) at (-5.5, -2.25) {$\mathtt{Reg}$};
    \node[complexity] (nc) at (-2.5, -2.25) {$\NC$};
    \node[complexity] (acd) at (0, -2.25) {$\ACd$};
    \node[complexity] (ac0) at (5.25, -2.25) {$\ACZero$};


    \draw[arrow_sep_right=black] (reg.east) -- (nc.west) {};

    \draw[arrow_sep_right=ETHRed] (reg.45) -- (mdm.220)
        node[target_label_style, above, sloped, pos=0.5] {$\log \strlen$ steps, $\strlen$ output}
        node[resource_style, sloped, below, pos=0.5] {\cref{thm:regular-languages-in-mdm-efficient}};

    \draw[arrow_sep_right=black] (cot.south) -- (ac0.north)
        node[source_label_style, above, sloped] {$\log \strlen$ steps};

    \draw[arrow_eq=black] (acd.west) -- (nc.east)
        node[label_style, above, sloped] {$d \to \infty$};

    \draw[arrow_eq=ETHRed] (acd.90) to
    node[target_label_style, right] {$\log^d \strlen$ steps \\ $\polyFun{\strlen}$ output}
    node[resource_style, right, yshift=-5mm, xshift=4mm] {\cref{cor:mdms-with-polylog-steps}}
    (mdm.270);

    \draw[arrow_eq=ETHRed] (ac0.105) to
    node[target_label_style, above, sloped, pos=0.5] {Constantly-many steps}
    node[target_label_style, below, sloped, pos=0.5] {$\polyFun{\strlen}$ output}
    node[resource_style, below=3mm, sloped, pos=0.5] {\cref{cor:mdms-with-constant-steps}}
    (mdm.320);

    \draw[arrow_sep_left=black] (ac0.west) -- (acd.east);

    \end{tikzpicture}

   \caption{
      A summary of masked diffusion model expressivity in relation to classical complexity classes.
      $\sX \hookrightarrow \sY$ indicates the inclusion of $\sX$ in $\sY$, and $\sX \leftrightarrow \sY$ indicates equality.
      Dashed arrows represent strict inclusions.
      \textcolor{ETHRed}{Red arrows} denote novel results.
      $\mathtt{Reg}$ refers to all regular languages.
      Edge labels with \colorbox{blue!10}{blue background} indicate constraints on the \emph{source} node, while labels with \colorbox{orange!10}{orange background} indicate constraints on the \emph{target} node (in case of $\leftrightarrow$, on \MDMsAcr).
      }
   \label{fig:figure-1b}
\end{figure}

\begin{figure}[ht]
    \centering
    \begin{tikzpicture}[
      scale=0.8,
      every node/.style={transform shape},
      box/.style={draw=ETHPetrol!50, fill=ETHPetrol!20, minimum width=5mm, minimum height=5mm, font=\footnotesize},
      maskbox/.style={box, fill=ETHGray!40, draw=ETHGray!80, dashed},
      mediumbox/.style={box, fill=ETHBronze!20, draw=ETHBronze!60},
      cotgenbox/.style={box, fill=ETHRed!40, draw=ETHRed!80, dashed},
      mdmgenbox/.style={box, fill=ETHGreen!40, draw=ETHGreen!60},
      treearrow/.style={-, thick, ETHGreen!60},
      seqarrow/.style={->, >=stealth, thick, ETHBlue!50},
      label/.style={font=\small},
      tinylabel/.style={font=\tiny\itshape, text=gray},
      brace/.style={decoration={brace,amplitude=5pt}, thick, decorate},
      ]
      
      \usetikzlibrary{decorations.pathreplacing, calc, positioning}
      
      \begin{scope}[shift={(0, -5)}]
          \node[box] (char0) at (0, 0) {$\scriptstyle ($};
          \node[box] (char1) at (0.5, 0) {$2$};
          \node[box] (char2) at (1, 0) {$+$};
          \node[box] (char3) at (1.5, 0) {$3$};
          \node[box] (char4) at (2, 0) {$\scriptstyle )$};
          \node[box] (char5) at (2.5, 0) {$\times$};
          \node[box] (char6) at (3, 0) {$\scriptstyle ($};
          \node[box] (char7) at (3.5, 0) {$4$};
          \node[box] (char8) at (4, 0) {$+$};
          \node[box] (char9) at (4.5, 0) {$1$};
          \node[box] (char10) at (5, 0) {$\scriptstyle )$};
          \node[box] (char11) at (5.5, 0) {$-$};
          \node[box] (char12) at (6, 0) {$6$};
          \node[box] (char13) at (6.5, 0) {$=$};
          
          \node[maskbox] (char14) at (7, 0) {\small \maskSym};
          \node[maskbox] (char15) at (7.5, 0) {\small \maskSym};
          \node[maskbox] (char16) at (8, 0) {\small \maskSym};
          \node[maskbox] (char17) at (8.5, 0) {\small \maskSym};
          \node[maskbox] (char18) at (9, 0) {\small \maskSym};
          \node[maskbox] (char19) at (9.5, 0) {\small \maskSym};
          \node[maskbox] (char20) at (10, 0) {\small \maskSym};
          \node[maskbox] (char21) at (10.5, 0) {\small \maskSym};
          \node[maskbox] (char22) at (11, 0) {\small \maskSym};
          \node[maskbox] (char23) at (11.5, 0) {\small \maskSym};
          \node[maskbox] (char24) at (12, 0) {\small \maskSym};
          
          \node[mdmgenbox] (char25) at (7, 0.75) {$5$};
          \node[mdmgenbox] (char26) at (7.5, 0.75) {$\times$};
          \node[mdmgenbox] (char27) at (8, 0.75) {$5$};
          \node[mdmgenbox] (char28) at (8.5, 0.75) {$-$};
          \node[mdmgenbox] (char29) at (9, 0.75) {$6$};
          \node[mdmgenbox] (char30) at (9.5, 0.75) {$=$};
          \node[maskbox] (char31) at (10, 0.75) {\small \maskSym};
          \node[maskbox] (char32) at (10.5, 0.75) {\small \maskSym};
          \node[maskbox] (char33) at (11, 0.75) {\small \maskSym};
          \node[maskbox] (char34) at (11.5, 0.75) {\small \maskSym};
          \node[maskbox] (char35) at (12, 0.75) {\small \maskSym};
          
          \node[mdmgenbox] (char36) at (7, 1.5) {$5$};
          \node[mdmgenbox] (char37) at (7.5, 1.5) {$\times$};
          \node[mdmgenbox] (char38) at (8, 1.5) {$5$};
          \node[mdmgenbox] (char39) at (8.5, 1.5) {$-$};
          \node[mdmgenbox] (char40) at (9, 1.5) {$6$};
          \node[mdmgenbox] (char41) at (9.5, 1.5) {$=$};
          \node[mdmgenbox] (char42) at (10, 1.5) {$25$};
          \node[mdmgenbox] (char43) at (10.5, 1.5) {$-$};
          \node[mdmgenbox] (char44) at (11, 1.5) {$6$};
          \node[mdmgenbox] (char45) at (11.5, 1.5) {$=$};
          \node[maskbox] (char46) at (12, 1.5) {\small \maskSym};
          
          \node[mdmgenbox] (char47) at (7, 2.25) {$5$};
          \node[mdmgenbox] (char48) at (7.5, 2.25) {$\times$};
          \node[mdmgenbox] (char49) at (8, 2.25) {$5$};
          \node[mdmgenbox] (char50) at (8.5, 2.25) {$-$};
          \node[mdmgenbox] (char51) at (9, 2.25) {$6$};
          \node[mdmgenbox] (char52) at (9.5, 2.25) {$=$};
          \node[mdmgenbox] (char53) at (10, 2.25) {$25$};
          \node[mdmgenbox] (char54) at (10.5, 2.25) {$-$};
          \node[mdmgenbox] (char55) at (11, 2.25) {$6$};
          \node[mdmgenbox] (char56) at (11.5, 2.25) {$=$};
          \node[mdmgenbox] (char57) at (12, 2.25) {$19$};
          
          \draw[seqarrow] (char19.north) to (char30.south);
          \draw[seqarrow] (char30.north) to (char41.south);
          \draw[seqarrow] (char41.north) to (char52.south);

          \draw[brace,draw=ETHGreen] ($(char57.north east) + (0.1,0.1)$) -- ($(char24.south east) + (0.1,-0.1)$) 
              node [midway, right=5pt, font=\small, text=ETHGreen] {Three unmasking Steps};
      \end{scope}
      
      \begin{scope}[shift={(2.5, -4.25)}]
          \node[mediumbox] (root) at (1.5, 1.5) {\tiny $-$};
          \node[mediumbox] (mult) at (0, 0.75) {\tiny $\times$};
          \node[mediumbox] (six) at (3.5, 0.75) {\tiny $6$};
          \node[mediumbox] (plus1) at (-1.5, 0) {\tiny $+$};
          \node[mediumbox] (plus2) at (1.5, 0) {\tiny $+$};
          
          \draw[treearrow] (root) -- (mult);
          \draw[treearrow] (root) -- (six);
          \draw[treearrow] (mult) -- (plus1);
          \draw[treearrow] (mult) -- (plus2);
          
          \draw[treearrow] (plus1) -- ($(char2)+(0,0.25)$);
          \draw[treearrow] (plus2) -- ($(char8)+(0,0.25)$);
          \draw[treearrow] (six) -- ($(char12)+(0,0.25)$);
          
          \node[text=ETHBronze,font=\small] at (1.5, 2) {Tree Structure};
      \end{scope}
      
      \begin{scope}[shift={(0, -5.75)}]
          \node[cotgenbox] (cot1) at (7, 0) {$5$};
          \node[cotgenbox] (cot2) at (7.5, 0) {$\times$};
          \node[cotgenbox] (cot3) at (8, 0) {$5$};
          \node[cotgenbox] (cot4) at (8.5, 0) {$-$};
          \node[cotgenbox] (cot5) at (9, 0) {$6$};
          \node (cot_dots) at (10.25, 0) {\dots};
          \node[cotgenbox] (cot_end_final) at (12, 0) {$19$};
          
          \node[label, text=ETHRed] at (4, 0) {Eleven chain of thought steps:};

          \draw[seqarrow, bend right=45] (cot1.south) to (cot2.south);
          \draw[seqarrow, bend right=45] (cot2.south) to (cot3.south);
          \draw[seqarrow, bend right=45] (cot3.south) to (cot4.south);
          \draw[seqarrow, bend right=45] (cot4.south) to (cot5.south);
      \end{scope}
    \end{tikzpicture}
    \caption{
    \textbf{Two strategies for solving a mathematical expression}. \emph{(a) \textcolor{ETHGreen}{Parallel}}: Parallel computation of intermediate values (three steps).
    \emph{(b) \textcolor{ETHRed}{Sequential}}: Step-by-step generation (eleven steps).
    }
    \label{fig:expression-tree}
\end{figure}

Our analysis builds upon the study of LM expressivity---the formal characterization of the problems whose solution the neural architecture of an LM (with appropriate parameters) can express. 
While this field comprehensively describes autoregressive LMs, its findings do not directly apply to \MDMsAcr due to their fundamentally different, non-sequential processing of text, leaving the theoretical studies of these two paradigms largely disconnected. 
Prior theoretical work on \MDMsAcr has focused on the limitations of their factorized backward process and its convergence properties.
This research has shown that while \MDMsAcr can approximate $n$-gram LMs with constantly many denoising steps, the number of steps must grow linearly with the string length even for simple LMs such as probabilistic regular languages \citep{feng2025theoreticalbenefitlimitationdiffusion,li2025convergencetheorydiffusionlanguage}.\footnote{We provide a detailed overview of related work in \cref{app:related-work}.} 
However, the asymptotic nature of these findings, together with the strict assumptions on the theoretical model, makes it difficult to draw concrete conclusions about the practical reasoning capabilities of \MDMsAcr, leaving a critical gap: A theoretical framework for \MDMsAcr that is both 
\begin{enumerate*}[label={(\arabic*)}]
    \item formally rigorous enough to provide a comprehensive picture of how \MDMsAcr can use the combination of parallelism and iterative refinement for formal reasoning, and
    \item faithful to how they are implemented in practice.
\end{enumerate*}
Our work introduces such a framework, providing a tight characterization of their reasoning capabilities.

Concretely, we connect \MDMsAcr implemented as finite-precision transformers with logarithmically-growing model width to known reasoning paradigms of \CoTAcr \citep{wei2022chainofthought}, looping \citep{dehghani2019universaltransformers}, and pause tokens \citep{lanham2023measuringfaithfulnesschainofthoughtreasoning}.
For example, we formalize:
\begin{takeaway}[\textit{\cref{thm:mdms-can-simulate-cot,thm:cot-can-simulate-mdms}}] \label{takeaway:takeaway-1}
   \MDMsAcr can perform \CoTAcr reasoning with some overhead and the \MDMAcr denoising process can be (inefficiently) simulated by \CoTAcr by generating one symbol at a time.
\end{takeaway}
We also show how \MDMsAcr can solve complex problems by solving easier sub-problems in parallel, which makes them inherently more efficient than \CoTAcr on parallelizable problems.
\begin{takeaway}[\textit{\cref{thm:cot-mdm-separation}}] \label{takeaway:takeaway-2}
   \MDMsAcr are provably more efficient at parallelizable problems than \CoTAcr.
\end{takeaway}
We refer to the fact that \CoTAcr \emph{cannot} take advantage of this parallelism as the \defn{sequentiality bottleneck} of \CoTAcr.
\cref{takeaway:takeaway-2} highlights the potential efficiency gains of \emph{parallel} \CoTAcr that generates multiple symbols at once.
Our analysis, in fact, identifies a tighter and more natural connection between \MDMsAcr and this variant, which we refer to as \pCoTAcr.
The parallelism of \MDMsAcr also facilitates a close connection with \emph{looped} and \emph{padded} transformers, where looping naturally maps to denoising steps and padding tokens to the generated tokens of an \MDMAcr.
We find that the class of problems solvable by \MDMsAcr is, in fact, precisely equivalent to that solvable by padded looped transformers.
\begin{takeaway}[\textit{\cref{thm:lts-mdms}}] \label{takeaway:takeaway-3}
   \MDMsAcr are equivalent to \LTsAcr.
\end{takeaway}

These connections allow us to leverage known characterizations of \CoTAcr and \LTsAcr together with classical complexity theory results to understand the fundamental capabilities and limits of \MDMsAcr \citep{li2024chain,saunshi2025reasoninglatentthoughtspower,london2025pausetokensstrictlyincrease,svete2025exactexpressivepowerfiniteprecision}.
For example, with $\strlen$ representing the length of the input string and with $\ACd$ for $d \in \N$ being the standard class of Boolean circuits with \andGate, \orGate, and \notGate gates and depth $\bigO(\log^d\strlen)$, we obtain:
\begin{takeaway}[\textit{\cref{thm:regular-languages-in-mdm-efficient}}] \label{takeaway:takeaway-4}
   \MDMsAcr with $\log\strlen$ denoising steps can recognize regular languages.
\end{takeaway}
\begin{takeaway}[\textit{\cref{cor:mdms-with-constant-steps,cor:mdms-with-polylog-steps}}] \label{takeaway:takeaway-5}
   For $d \in \N$, \MDMsAcr with $\bigO(\log^d\strlen)$ denoising steps and $\polyFun{\strlen}$ output space are equivalent to $\ACd$.
   As $d \to \infty$, this yields $\NC$, the class of all parallelizable problems.
   With \emph{constantly} many denoising steps ($d = 0$), \MDMsAcr are equivalent to the limited class $\ACZero$, i.e., they are no more powerful than standard transformers and, e.g., cannot recognize all regular languages.
\end{takeaway}
We summarize these takeaways and their relationships in \cref{fig:figure-1a,fig:figure-1b}.
Our proofs emphasize the affordances and difficulties in solving reasoning problems with \MDMsAcr. 
We find, for example, that the discrete nature of the generated text, which serves as a communication channel between individual denoising steps, limits the amount of information that can be passed between steps.
This necessitates extra output space to store the intermediate computations and is analogous to the discrepancy between fixed-precision and log-precision transformers \citep{li2024chain}, and the difference between the classes $\AC$ and $\TC$, the class of circuits with threshold gates.
We find that \emph{positional encodings} that carry information not computable by transformers themselves are crucial in locating this information.
We also observe that, while unmasked attention makes \MDMAcr attention patterns flexible, it complicates left-to-right processing, which is often natural in human language---we find that exact implementation of causal masking requires quadratically more output space.

\section{Preliminaries} \label{sec:preliminaries}

This section introduces the preliminaries and notation used throughout the paper.
We reserve the main text for the high-level intuitions and defer technical details to \cref{app:preliminaries}.

\subsection{Transformers}

We analyze (un)masked transformers for string generation and classification.
We present the full model in \cref{app:transformers} and focus here on the most relevant aspects.

\textbf{Transformer families.}
We study \emph{finite precision} transformers where the value of each parameter and activation is represented with a fixed number of bits.
We allow the model \emph{width}, i.e., the size of the contextual representations, to grow logarithmically with the input length $\strlen$.
This is a standard assumption in the literature on transformer expressivity \citep{li2024chain} since it is necessary and sufficient for the model to uniquely identify input positions, and aligns with modern implementations of \emph{quantized} but \emph{wide} transformers.
The growing width results in a separate transformer $\tf_\strlen$ for each $\strlen$, yielding a \defn{family} $\{\tf_\strlen\}_{\strlen \in \N}$ of transformers. 
We enforce $\LUniformity$ in the family by requiring an associated Turing machine that constructs $\tf_\strlen$ in $\bigOFun{\log\strlen}$ space (cf. \cref{app:transformers}).

\textbf{(Parallel) \CoTAcr transformers.}
\CoTAcr reasoning enables sequential processing by solving problems in multiple steps \citep{wei2022chainofthought}.
It is an integral part of today's popular ``reasoning'' models and substantially increases transformers' expressivity \citep{li2024chain,merrill2024the}.
We define our idealization of \CoTAcr transformers in \cref{app:transformers}, including \emph{parallel} \CoTAcr transformers,
which predict $\padlen' \in \N$ symbols in parallel at each time step,
enabling some parallelism if the task allows it.\footnote{This is different from \emph{speculative decoding} \citep{10.5555/3618408.3619203}, which generates multiple symbols \emph{one a time} with a smaller LM and then evaluates their probability in a single pass of a larger LM.}
We denote the classes of \CoTAcr and parallel \CoTAcr transformers by $\CoTClass$ and $\pCoTClass$, respectively.

\textbf{Padded looped transformers (\LTsAcr).} 
Looped transformers repeatedly apply a fixed block of transformer layers to the input \citep{dehghani2019universaltransformers}.
This dynamically increases the depth of the model, enabling more complex reasoning, and does not increase the model size, as the same blocks are reused, thus reducing the memory footprint and computational cost \citep{bae2025mixtureofrecursionslearningdynamicrecursive}.
Such reasoning steps include both sequential and parallel processing, resulting in both efficiency as well as depth of the reasoning process.
Padded transformers additionally pad the input with blank symbols, which can be used to perform additional computations \emph{in parallel}.
This additional padding space is analogous to increasing the circuit width in circuit complexity.
We additionally provide \LTsAcr with external noise applied to the residual stream at each step, which enables stochastic computations
(cf.~\cref{app:looped-padded-transformers}).
We denote the class of padded looped transformers by $\LTClass$.

\textbf{Transformer LMs.}
An \defn{alphabet} $\alphabet$ is a finite, non-empty set of \defn{symbols}. 
Its Kleene closure is $\strings \defeq \bigcup_{n=0}^{\infty} \alphabet^n$, the set of all strings.
A \defn{language model} is a distribution over $\strings$.
Most LMs are \defn{autoregressive}---they define \defn{next-symbol} distributions $\pLMFun{\cdot \mid \str}$ over $\eosalphabet \defeq \alphabet \cup \set{\eos}$ for $\str \in \kleene{\alphabet}$, where $\eos \notin \alphabet$ is the \underline{e}nd-\underline{o}f-\underline{s}tring symbol.
A transformer-based LM computes $\pLMFun{\cdot \mid \str}$ by linearly transforming the contextual representation of the final symbol to the logits of a distribution over $\eosalphabet$.
Moreover, contextual representations can be used for \defn{infilling}---predicting symbols at masked positions.
Infilling probabilities $\pLMInfFun{\cdot \mid \str}$ at masked positions in $\str \in \kleene{\maskAlphabet}$ are distributions over $\alphabet$, where $\maskSym \notin \alphabet$ is the mask symbol and $\maskAlphabet \defeq \alphabet \cup \set{\maskSym}$.

\textbf{Transformers and formal languages.}
Plenty of work describes transformer capabilities and limitations with formal languages \citep{strobl-etal-2024-formal}.
These studies typically frame transformers as \defn{language recognizers}, i.e., classifiers that decide whether a string $\str \in \kleene{\alphabet}$ belongs to some formal language $\lang \subseteq \kleene{\alphabet}$ \citep{butoi2025training}.
String membership is usually \emph{deterministic} and can be formalized by determinizing the LM defined by a transformer: The next-symbol and infilling probabilities are used to decode the most probable symbol or decision.\footnote{Some work also considers transformers as LMs directly \citep{svete-cotterell-2024-transformers,nowak-etal-2024-representational,borenstein-etal-2024-languages,svete-etal-2024-transformers} and shows the (probabilistic) gains afforded by \CoTAcr reasoning.}
The final prediction of $\ind{\str \in \lang} \in \{\acceptSym, \rejectSym\}$ can be made 
\begin{enumerate*}[label=(\arabic*)]
    \item \emph{in a single pass} by classifying based on the contextual representation of a particular symbol in the string, analogous to classifying based on the \texttt{CLS} symbol in BERT \citep{devlin-etal-2019-bert}, or
    \item after \emph{``reasoning''}, i.e., solving the problem in multiple time steps. 
    In this case, the transformer's prediction is only made after a sequence of intermediate predictions that augment its computation and help the final decision.
    This is analogous to using \CoTAcr reasoning for string recognition and can be thought of as simulating a Turing machine with each step of the \CoTAcr process.
\end{enumerate*} 

Particularly fruitful has been the study of transformers as \defn{Boolean circuits}.
In particular, our idealization of transformers falls under $\ACZero$ circuits \citep{li2024chain}---Boolean circuits of constant depth, polynomial size, and with \andGate, \orGate, and \notGate gates of unbounded fan-in---and captures the entire class if padding is allowed \citep{london2025pausetokensstrictlyincrease}.
Other idealizations of transformers can compute functions outside of $\ACZero$ \citep{li2024chain,merrill2025exactexpressivepowertransformers} but remain in $\TCZero$, the class of \emph{threshold circuits} which add threshold gates (which determine whether the number of inputs exceeds some threshold) to $\ACZero$ circuits.
\cref{app:circuit-complexity,app:transformers} provide more details on circuit classes and their relation to transformers.

\subsection{Masked diffusion language models} \label{sec:diffusion-language-models}
\defn{Discrete diffusion LMs} define a distribution over $\strings$ by progressively denoising noisy strings sampled from some fixed distribution.
Formally, they define a \defn{forward (noising) process} and a \defn{reverse (denoising) process}.
The forward process defines a Markov chain over strings that iteratively corrupts them. 
Common examples include replacing symbols uniformly at random (uniform diffusion) or masking them with the mask symbol $\maskSym$ (\defn{masked diffusion models}, \MDMsAcr).
The latter is the focus of this work.
In this setting, the forward process starts from an initial string $\str^{(0)} \in \strings$ of some pre-determined length $\padlen$ and, at each of the $\timesteps$ (discrete) steps, independently masks symbols with probability determined by a masking schedule $\maskScheduleFun{\frac{t}{\timesteps}} \in [0, 1]$:
\begin{minipage}[t]{0.45\textwidth}
   \begin{equation*}
      \diffusionProcess_{t|0}(\str^{(t)} \mid \str^{(0)}) = \prod_{\stridx=1}^{\strlen} \diffusionProcess_{t|0}(\sym^{(t)}_\stridx \mid \sym^{(0)}_\stridx),
   \end{equation*}
\end{minipage}
\hfill
\begin{minipage}[t]{0.5\textwidth}
   \begin{equation*}
      \diffusionProcess_{t|0}(\sym^{(t)}_\stridx \mid \sym^{(0)}_\stridx) =
      \begin{cases}
         1 - \maskSchedule(\frac{t}{\timesteps}), &\ifcondition\sym^{(t)}_\stridx = \maskSym \\
         \maskSchedule(\frac{t}{\timesteps}), &\otherwisecondition
      \end{cases}.
   \end{equation*}
\end{minipage}
The masking schedule is set such that $\maskScheduleFun{0} = 1$ (no masking at the start) and $\maskScheduleFun{1} = 0$ (fully masked at the end, meaning that the noise distribution is the Dirac delta on the fully masked string).

Starting from the fully masked input, the reverse process $\diffusionProcess_{0|\timesteps}$ inverts the forward process $\diffusionProcess_{\timesteps|0}$ by 
\begin{enumerate*}[label=(\arabic*)]
   \item uniformly selecting some positions to unmask, and
   \item sampling the chosen unmasked symbols.
\end{enumerate*}
After $\timesteps$ denoising steps, $\diffusionProcess_{0|\timesteps}$ produces a string $\str^{(0)}$ sampled from the LM defined by the diffusion process.
It is this reverse process that is learned from data.
Its analytical form is generally intractable, so one usually models a parameterized approximation of a single denoising step $\learnedGenerativeProcess_{t - 1|t}(\str^{(t-1)} \mid \str^{(t)})$, typically implemented as a transformer, that \emph{factorizes} across positions:
\begin{equation} \label{eq:factorized-reverse-process}
   \learnedGenerativeProcess_{t - 1|t}(\str^{(t-1)} \mid \str^{(t)}) = \prod_{\stridx=1}^\strlen \learnedGenerativeProcess_{t - 1|t}(\sym^{(t-1)}_\stridx \mid \str^{(t)})
\end{equation}
\cref{eq:factorized-reverse-process} enables \emph{parallel generation} but ignores inter-symbol dependencies at each denoising step.

Much of the existing work on \MDMAcr expressivity analyzes the convergence of $\learnedGenerativeProcess_{t - 1|t}$ to $\diffusionProcess_{0|\timesteps}$ \citep{li2025convergencetheorydiffusionlanguage,chen2024convergenceanalysisdiscretediffusion,feng2025theoreticalbenefitlimitationdiffusion}.
Studying convergence properties usually requires assuming uniform unmasking and a good approximation of the ground-truth model \citep[e.g.,][]{li2025convergencetheorydiffusionlanguage,chen2024convergenceanalysisdiscretediffusion,feng2025theoreticalbenefitlimitationdiffusion,liu2025perfectdiffusionmathsftc0}.
\begin{assumption}[Uniform unmasking] \label{assumption:uniform-unmasking}
   The \defn{uniform unmasking assumption} states that $\learnedGenerativeProcess_{t - 1|t}$ generates strings $\str^{(t - 1)}$ from $\str^{(t)}$ by uniformly selecting positions to unmask.
\end{assumption}
\begin{assumption}[Perfect approximation] \label{assumption:perfect-approximation}
   Let $\qdens_{t - 1|t}$ be the backward processes of an \MDMAcr and $\pLMInf$ a transformer-based LM.
   The \defn{perfect approximation assumption} states that
   $\qLM_{t - 1|t}(\sym^{(t - 1)}_\stridx \mid \str^{(t)}) = \pLMInfFun{\sym^{(t - 1)}_\stridx \mid \str^{(t)}}$
   for all $\str^{(t)} \in \kleene{\maskAlphabet}$, $\sym^{(t - 1)}_\stridx \in \alphabet$, and $\stridx \in \{1, \ldots, |\str^{(t)}|\}$.
\end{assumption}
In words, \cref{assumption:perfect-approximation} states that the transformer perfectly models \emph{all} conditional distributions of the diffusion process.\footnote{Or \emph{approximates} them well, requiring the error to be smaller than some $\epsilon > 0$.}
While this seems necessary, the following observation, proved in \cref{app:proofs}, shows that \cref{assumption:perfect-approximation,assumption:uniform-unmasking} severely limit the class of functions that the model can compute.
\begin{restatable}[]{theorem}{perfectApproximationTheorem} \label{thm:perfect-approximation-limitation} 
   If \cref{assumption:perfect-approximation,assumption:uniform-unmasking} hold for an LM $\pdens$, $\pdens$ cannot compute non-$\ACZero$ functions.\footnote{That is, $\pdens$ can only implement LMs whose next-symbol logits can be computed by $\ACZero$ circuits \citep{liu2025perfectdiffusionmathsftc0}.}\textsuperscript{,}\footnote{An analogous version of the theorem applies to transformers in $\TCZero$.}
\end{restatable}
By assuming that the \MDMAcr is unable to choose which positions to unmask, the model has no choice in which sub-problems to solve first, which ignores the possibility of problem decomposition and requires the model to be equally good at solving \emph{any} sub-problem---including predicting the final answer based on the input directly (with no reasoning steps).
This implies that the problem is solvable in a single prediction step of a transformer.
However, the expressivity of a single transformer pass is limited---\cref{thm:perfect-approximation-limitation} uses the fact that fixed-depth transformers lie in $\ACZero$ \citep{li2024chain}.
However, not \emph{all} conditional probabilities have to be known to be able to solve algorithms in few steps.
Intuitively, by \emph{choosing} to solve \emph{simple} subproblems with non-random unmasking, an \MDMAcr can avoid the difficult parts.
For example, given the current arithmetic expression, one only has to predict the next set of simplifications---which are simple functions of the current expression.
This motivates us to \emph{loosen} \cref{assumption:uniform-unmasking,assumption:perfect-approximation}, which we do in our idealization of an \MDMAcr.
   
\subsubsection{Our idealization of masked diffusion models} \label{sec:idealization}

We aim to understand the expressivity of the \emph{reverse process}. 
To this end, we introduce an idealization that captures its key aspects---iterative unmasking and infilling---and provides a principled lens for understanding the expressivity of practical \MDMsAcr and comparing them to well-known paradigms such as \CoTAcr. 
Here, we describe the high-level ideas; see \cref{app:idealization} for the full formal model.

We formalize the reverse process with two components: A \defn{planner} that \emph{decides} which positions to unmask at each step and a \defn{predictor} that samples the symbols at the unmasked positions.
This loosens \cref{assumption:uniform-unmasking} and generalizes standard \MDMsAcr in which the planner is implicitly defined by choosing the positions to unmask uniformly at random.
It also mirrors popular \MDMAcr implementations that generate text by selecting a subset of masked positions---for example, based on model confidence or according to a learned policy---and predicting the symbols conditioned on the current partially unmasked string \citep{ghazvininejad-etal-2019-mask,peng2025pathplanningmaskeddiffusion,zheng2024reparameterizeddiscretediffusionmodel,liu2025thinkgeneratediscretediffusion,kim2025trainworstplanbest,benhamu2025acceleratedsamplingmaskeddiffusion}.\footnote{\cref{thm:single-model-equivalence} in \cref{app:planner-centrality} shows that the planner and predictor can be fused into a single model that unmasks symbols based on their confidence.
This means that all our results apply to this popular model of unmasking.}
Discarding \cref{assumption:uniform-unmasking} also sidesteps limitations of the position-wise independence in~\cref{eq:factorized-reverse-process}, a restriction that prevents \MDMsAcr with uniform unmasking matching even simple distributions exactly \citep{feng2025theoreticalbenefitlimitationdiffusion,wu2025fastdllmtrainingfreeaccelerationdiffusion}.
In reasoning problems with a deterministic sequential structure, the ability to decide what to unmask enables problem decomposition into a sequence of deterministic steps that can be solved in parallel.\footnote{This is related to the importance of the mutual information and correlations between symbols to \MDMAcr performance \citep{li2025convergencetheorydiffusionlanguage,wu2025fastdllmtrainingfreeaccelerationdiffusion} and string-level correctness \citep{feng2025theoreticalbenefitlimitationdiffusion}.}
To connect our analysis to practical implementations, we assume that the planner and predictor are implemented as transformers.\looseness=-1

We allow the planner to choose to \emph{resample} already unmasked positions.
This overcomes another key limitation---the inability to revert decisions and correct earlier mistakes---a challenge that is the focus of much recent research \citep[][\textit{inter alia}]{rutte2025generalized,song2025seeddiffusionlargescalediffusion}.\footnote{While our results focus on \MDMsAcr that can resample generated symbols, they also apply to non-resampling \MDMsAcr---\cref{thm:mdms-can-simulate-emdms} in \cref{app:editing-mdms} shows that the latter can simulate the former if given additional output space.}
While existing work focuses on reformulating the diffusion process to allow for resampling and refining the resulting training objectives, our idealized \MDMsAcr can be seen as a complementary approach that foregoes the complications of training and focuses on the expressivity of the generation process itself---it analyzes what is theoretically possible in a very targeted way when resampling is allowed.

We denote the class of our idealized \MDMsAcr by $\MDMClass$.
For a more detailed discussion, including the theoretical connection of our idealization to existing \MDMAcr variants, see \cref{app:theoretical-model}.

\section{Theoretical results}
\label{sec:theoretical-results}

This section describes two complementary characterizations of \MDMAcr expressivity: One based on their connection to \LTsAcr and the other based on their ability to perform sequential \CoTAcr reasoning.\footnote{The proofs of all statements in this section are deferred to \cref{app:proofs}.}

\paragraph{Notation.}
Let $\timesteps = \timesteps(\strlen)$ and $\padlen = \padlen(\strlen)$ be functions of $\strlen$.
In the following, $\CoTClass[\timesteps]$ refers to languages recognized by families $\{\tf_\strlen\}_{\strlen \in \N}$ of \CoTAcr transformers with at most $\timesteps$ steps.
For $\modelClass \in \{\pCoTClass,\MDMClass,\LTClass\}$, $\modelClass[\timesteps, \padlen]$ refers to languages recognized by families in $\modelClass$ with at most $\timesteps$ generation, denoising, or looping steps, respectively, and $\padlen$ total output or padding symbols.
$\bigOLogFun{\strlen}$ refers to big-$\bigO$ notation that ignores logarithmic factors, and $\polyFun{\strlen}$ to polynomial functions in $\strlen$.

\subsection{\MDMsAcr are equivalent to padded looped transformers} \label{sec:mdms-and-lts}

Intuitively, \LTsAcr closely resemble \MDMsAcr: Both iteratively refine information in parallel---\MDMsAcr by unmasking and predicting discrete symbols, and \LTsAcr by updating the residual stream.\footnote{\LTsAcr thus also resemble \emph{latent} diffusion LMs that diffuse in the representation space. We do not explore this connection here due to the superior performance and popularity of \MDMsAcr \citep{zhang2025surveyparalleltextgeneration} in practice.}
\cref{thm:lts-mdms} formalizes this, assuming \LTsAcr are supplied with external sampling noise (cf.~\cref{sec:preliminaries}, \S\ref{app:looped-padded-transformers}) like \MDMsAcr.

\begin{restatable}[\LTsAcr and \eMDMsAcr ]{theorem}{LTsandMDMsThm} \label{thm:lts-mdms} \hfill
   
   \begin{subequations}      
      \noindent\begin{minipage}[b]{0.4\textwidth}
         \begin{equation}
            \MDMClassEdit[\timesteps, \padlen] \subseteq \LTClass[\timesteps, \padlen]
         \end{equation}
      \end{minipage}
      \hfill
      \begin{minipage}[b]{0.05\textwidth}
         \centering
         and
      \end{minipage}
      \hfill
      \begin{minipage}[b]{0.45\textwidth}
         \begin{equation} \label{eq:mdm-simulates-lt}
            \LTClass[\timesteps, \padlen] \subseteq \MDMClassEdit[\timesteps, (\strlen + \padlen) \textcolor{ETHRed}{\hiddDim}].
         \end{equation}
      \end{minipage}
   \end{subequations}
\end{restatable}
The simulation of a \LTAcr by an \MDMAcr incurs a factor $\hiddDim$ increase in the required padding (cf.~\cref{eq:mdm-simulates-lt}), where $\hiddDim$ is the model width of the \LTAcr.
In our setting where $\hiddDim = \bigOFun{\log \strlen}$, this implies that the classes of finite-precision \MDMsAcr and \LTsAcr coincide up to a logarithmic factor in the padding length:
\begin{restatable}{corollary}{MDMEquivalenceCor} \label{cor:mdm-lt-equivalence}
   For any $K \geq 1$,
   \begin{equation}
      \MDMClassEdit[\timesteps, \bigOLogFun{\strlen^K}] = \LTClass[\timesteps, \bigOFun{\strlen^K}].
   \end{equation}
\end{restatable}

The close connection between \MDMsAcr and \LTsAcr allows us to leverage existing results about \LTAcr expressivity to understand \MDMsAcr.
\citet[][Thm. 5.1]{saunshi2025reasoninglatentthoughtspower}, for example, show that log-depth unpadded transformers can recognize regular languages.
Combined with \cref{cor:mdm-lt-equivalence}, this implies:
\begin{restatable}{corollary}{RegularLanguagesInMDMLem} \label{cor:mdms-can-simulate-regular-languages}
   Regular languages are in $\MDMClass[\log\strlen, \strlen \log\strlen]$.
\end{restatable}

In fact, we obtain a tighter bound with a more specialized construction.\footnote{This is possible because the \LTAcr of \cite{saunshi2025reasoninglatentthoughtspower} only stores discrete values in its residual stream.}
\begin{restatable}{theorem}{RegularLanguagesInMDMEfficientThm} \label{thm:regular-languages-in-mdm-efficient}
   Regular languages are in $\MDMClass[\log\strlen, \strlen]$.
\end{restatable}

\citet{london2025pausetokensstrictlyincrease} show that polynomially padded finite-precision \LTsAcr with constantly many steps are equivalent to $\LUniform \ \ACZero$, the class of $\ACZero$ circuits that can be constructed by a logspace Turing machine (cf.~\cref{app:circuit-complexity}). 
We leverage this result together with \cref{thm:lts-mdms} to characterize the expressivity of \MDMsAcr with constantly many denoising steps.
\begin{restatable}[\MDMsAcr with constantly many denoising steps]{corollary}{MDMsConstantStepsCor} \label{cor:mdms-with-constant-steps}
   \begin{equation}
      \MDMClassEdit[\bigOFun{1}, \polyFun{\strlen}] = \LTClass[\bigOFun{1}, \polyFun{\strlen}] = \LUniform \ \ACZero.
   \end{equation}
\end{restatable}
Allowing for a \emph{constant} number of decoding steps therefore does not increase the expressivity of \MDMsAcr beyond the limited class $\ACZero$.
This further corroborates the empirical observation that the number of denoising steps must scale with the input complexity and complements existing results on \MDMAcr expressivity as a function of the number of denoising steps \citep{li2025convergencetheorydiffusionlanguage,feng2025theoreticalbenefitlimitationdiffusion}.

We can, however, increase expressivity with more denoising steps.
\citet{svete2025exactexpressivepowerfiniteprecision} show that finite-precision \LTsAcr with $\bigOFun{\log^d \strlen}$ steps and polynomial padding are equivalent to $\LUniform \ \ACd$, similar to the case of log-precision \LTsAcr and $\LUniform \ \TCd$ \citep{merrill2025exactexpressivepowertransformers}.
Thus:
\begin{restatable}[\MDMsAcr with polylogarithmically many denoising steps]{corollary}{MDMsPolylogStepsCor} \label{cor:mdms-with-polylog-steps}
   \begin{equation}
      \MDMClassEdit[\bigOFun{\log^d \strlen}, \polyFun{\strlen}] = \LTClass[\bigOFun{\log^d \strlen}, \polyFun{\strlen}] = \LUniform \ \ACd
   \end{equation}
\end{restatable}
In particular, since $\NCFun{d} \subseteq \ACFun{d}$ for $d \in \N$ (where $\NCFun{d}$ denotes $\ACd$ circuits with bounded fan-in) \citep{Vollmer1999}, we get that with polylogarithmic looping and polynomial padding, \MDMsAcr converge to $\NC$, the class of all parallelizable problems.
\cref{cor:mdms-with-polylog-steps} also implies that regular languages are in $\MDMClassEdit[\log \strlen, \polyFun{\strlen}]$; \cref{thm:regular-languages-in-mdm-efficient} provides a more efficient construction with a linear output space.
Moreover, \cref{cor:mdms-with-polylog-steps} implies that $\LUniform \ \NCOne \subseteq \MDMClassEdit[\log \strlen, \polyFun{\strlen}]$.

\subsection{\MDMsAcr and \CoTAcr can (inefficiently) simulate each other} \label{sec:mdms-and-cot}  

While the close connection between \MDMsAcr and \LTsAcr provides a useful lens to analyze \MDMsAcr in terms of known complexity classes, the lack of practical \LTAcr implementations makes it difficult to draw intuitive conclusions.
We therefore complement \cref{sec:mdms-and-lts} by connecting \MDMsAcr and the more popular \CoTAcr paradigm,
and consider how \MDMAcr can behave ``autoregressively'' like \CoTAcr.
The intuition is simple: An \MDMAcr can simulate \CoTAcr by unmasking one symbol at a time, effectively mimicking the sequential generation.
More precisely, we connect \MDMsAcr to \pCoTAcr since the latter's ability to generate multiple symbols at once naturally maps to \MDMsAcr' parallel generation.
In particular:
\begin{restatable}[\eMDMsAcr  can simulate \pCoTAcr transformers]{theorem}{MDMsAtLeastAsPowerfulAsCoTFP} \label{thm:mdms-can-simulate-cot}
   \begin{equation}
      \pCoTClass[\timesteps, \padlen] \subseteq \MDMClassEdit[\timesteps, \padlen + (\strlen + \padlen)^2]
   \end{equation}
\end{restatable}
The simulation incurs a quadratic blow-up in the padding length.
This is not due to an inherent feature of the diffusion process but rather the challenge of simulating masked attention with unmasked one (cf.~\cref{lem:unmask-conversion-fp}).\footnote{To the best of our knowledge, \cref{lem:unmask-conversion-fp} is the first result showing how to simulate masked attention with unmasked attention and might be of interest in its own right.} 
In particular, if the \MDMAcr transformer is causally masked, the blow-up disappears. 
Moreover, if the unmasked transformer can simulate masking more efficiently (with, for example, more expressive scoring functions), the blow-up can be alleviated.\footnote{This is, for example, used by \citet[][Thm 5.4]{saunshi2025reasoninglatentthoughtspower} to show that unmasked \LTsAcr can simulate \CoTAcr with no blow-up in the padding using a \emph{masking function}---an additional step in the computation of attention scores that allows for zeroing out of irrelevant keys---and with linearly-increasing width.}
While \cref{lem:unmask-conversion-fp} could possibly be improved, it is interesting to note that the seemingly more general, unmasked
nature of \eMDMsAcr might negatively impact their ability to align with human-oriented sequential processing captured by causal masking, which could provide a useful inductive bias for the masked models.\footnote{Interestingly, some existing work finds that \MDMsAcr that decode based on the most confident symbols tend to decode autoregressively \citep{gong2025diffucoderunderstandingimprovingmasked}. In this sense, the unmasked nature of \MDMsAcr could be seen as a hurdle that the model has to overcome to eventually rely on more autoregressive generation. This further motivates the development of hybrid models that combine autoregressive generation of entire blocks with non-autoregressive infilling within the blocks \citep[][\textit{inter alia}]{nie2025largelanguagediffusionmodels,arriola2025blockdiffusioninterpolatingautoregressive,song2025seeddiffusionlargescalediffusion}.}

The other direction of \cref{thm:mdms-can-simulate-cot} shows that \pCoTAcr transformers can simulate \eMDMsAcr .
\begin{restatable}[\pCoTAcr transformers can simulate \eMDMsAcr ]{theorem}{MDMsCoTChain} \label{thm:cot-can-simulate-mdms}
   \begin{equation}
      \MDMClassEdit[\timesteps, \padlen] \subseteq \pCoTClass[\timesteps, \numlayers \timesteps (\padlen + \strlen)], 
   \end{equation}
   where $\numlayers$ is the number of layers in the transformer implementing the \eMDMAcr.
\end{restatable}
Again, the factor $\numlayers$ comes from the need to simulate unmasked attention in \MDMsAcr with causally masked transformers implementing \pCoTAcr (cf.~\cref{lem:mask-conversion}).
However, here, the blow-up is only linear.
The additional factor of $\timesteps$ comes from the \pCoTAcr having to write out every padding token after each denoising step, as the \MDMAcr can unmask tokens in an arbitrary order.

The results above can be summarized by the following sequence of inclusions.
\begin{restatable}{corollary}{MDMCoTSandwich} \label{thm:mdms-cot-sandwich} We have the following set of inclusions:
   \begin{subequations}
      \begin{align}
         \CoTClass[\timesteps] &= \pCoTClass[\timesteps, \timesteps] \justification{\cref{prop:cot-as-pcot}} \\ 
         &\subseteq \MDMClassEdit[\timesteps, \timesteps + (\strlen + \timesteps)^2] \justification{\cref{thm:mdms-can-simulate-cot}} \\ 
         &\subseteq \pCoTClass[\timesteps, \numlayers \timesteps(\strlen + \timesteps + (\strlen + \timesteps)^2)] \justification{\cref{thm:cot-can-simulate-mdms}} \\ 
         &\subseteq \CoTClass[\numlayers \timesteps (\strlen + \timesteps + (\strlen + \timesteps)^2)] \justification{\cref{prop:cot-as-pcot}} \\
         &\subseteq \CoTClass[\numlayers \timesteps (\strlen + \timesteps + 1)^2].
      \end{align}
   \end{subequations}
   In particular, when $\timesteps \geq \strlen$, we have
      $\CoTClass[\timesteps]
      \subseteq \MDMClassEdit[\timesteps, \bigOFun{\timesteps^2}]
      \subseteq \pCoTClass[\timesteps, \bigOFun{\timesteps^3}]
      \subseteq \CoTClass[\bigOFun{\timesteps^3}]$.
\end{restatable}
\cref{thm:mdms-cot-sandwich} lower- and upper-bounds \eMDMAcr expressivity based on the expressivity of \CoTAcr transformers.
For example, \eMDMAcr with polynomially many denoising steps remain within the class $\polyCls$, the problems solvable in polynomial time by a non-random-access multitape Turing machine.
This follows from the equivalence of \CoTAcr transformers with polynomially many steps to $\polyCls$ \citep{li2024chain}.
\begin{corollary}[\MDMsAcr with polynomially many denoising steps] 
   For any $K \geq 1$, we have that
   \begin{equation}
      \MDMClassEdit[\timesteps, \strlen^K] \subseteq \CoTClass[\bigOFun{\timesteps \strlen^K}],
   \end{equation}
   meaning that \eMDMsAcr  with polynomially many denoising steps remain in $\polyCls$:
   \begin{equation}
      \MDMClassEdit[\polyFun{\strlen}, \polyFun{\strlen}] \subseteq \CoTClass[\polyFun{\strlen}] \subseteq \polyCls.
   \end{equation}
\end{corollary}

\subsubsection{A separation between \eMDMsAcr  and \CoTAcr transformers}
\citet{merrill2024the,li2024chain} show that \CoTAcr transformers with logarithmically many decoding steps remain in $\TCZero$.
Combining this with \cref{thm:regular-languages-in-mdm-efficient}, the widely accepted assumption that $\TCZero \neq \NCOne$, and known $\NCOne$-completeness of specific regular languages, we obtain the following separation in expressivity under a small (logarithmic) number of decoding steps.
We term the inability of \CoTAcr to leverage parallelism the \defn{sequentiality bottleneck} of \CoTAcr.
\begin{restatable}[A strict separation in efficient reasoning abilities of \eMDMsAcr and \CoTAcr transformers]{corollary}{CoTMDMSeparationCor} \label{thm:cot-mdm-separation}
   \begin{equation}
      \CoTClass[\log\strlen] \subsetneq \MDMClass[\log\strlen, \strlen].
   \end{equation}
\end{restatable}
Concretely, $\MDMClass[\log\strlen, \strlen] \, \setminus \, \CoTClass[\log\strlen]$, for example, contains all $\NCOne$-complete regular languages.

\section{Discussion} \label{sec:discussion}

\paragraph{Strengths and weaknesses of \MDMsAcr.}
\cref{sec:theoretical-results} provides insights into the suitability of using \MDMsAcr for different classes of problems.
On the one hand, \cref{thm:cot-mdm-separation} reveals the sequentiality bottleneck of \CoTAcr and an expressivity gap between \CoTAcr transformers and \MDMsAcr with logarithmically many model evaluations: While \CoTAcr transformers remain in $\TCZero$, \MDMsAcr can solve $\NCOne$-complete problems.
This formalizes the intuition that \MDMsAcr are more suitable for highly-parallelizable problems and has implications for the practical applications of these two paradigms with a limited number of model evaluations.
For example, the common \emph{state-tracking} benchmark used to evaluate the reasoning abilities \citep{liu2023transformers,10.5555/3692070.3693514} can be solved with \MDMsAcr with logarithmically many steps, while \CoTAcr transformers require linearly many steps.
On the other hand, the equivalence of \MDMsAcr with polylogarithmically many denoising steps to the class $\NC$ (cf.~\cref{cor:mdms-with-polylog-steps}) reveals problems where efficiency gains from parallelism are limited.
For example, assuming the widely-believed hypothesis that $\NC \neq \polyCls$, none of the following ($\polyCls$-complete) problems benefit from \MDMAcr parallelism:
\begin{itemize}[noitemsep,topsep=0pt,leftmargin=10pt]
   \item \textbf{Circuit value problem}: Given a circuit, its inputs, and a gate, calculate the gate's value.
   \item \textbf{Linear programming}: Maximize a linear function subject to linear inequality constraints.
   \item \textbf{Context free grammar (CFG) membership}: Given a CFG $\grammar$ and a string $\str$, is $\str \in \lang(\grammar)$?
   \item \textbf{Horn-satisfiability} ($\polyCls$ version of \texttt{SAT}): Is there a satisfying assignment to a set of Horn clauses?
\end{itemize}
In other words, these problems, in general, require a ``\CoTAcr-style'' step-by-step sequential solution. 
Due to the overhead introduced by unmasked attention of \MDMsAcr (such as the inability to maintain KV-cache), such problems are more efficiently solved by standard autoregressive \CoTAcr transformers.

\paragraph{Equivalence to padded looped transformers.}
\cref{cor:mdm-lt-equivalence} reveals a tight connection between \MDMsAcr and \LTsAcr.
While this suggests these two frameworks are largely interchangeable, important distinctions exist.
On the one hand, unlike \MDMsAcr, standard \LTsAcr perform sequential computations \emph{deterministically}.
This makes \MDMsAcr more suitable for ambiguous generation tasks, where decisions early in the generation make subsequent decisions easier.
\LTsAcr would, in that case, have to keep track of all possible generations in the residual stream until the final---decoding---step.
Moreover, \MDMsAcr are easier to train and steer---since their intermediate computation steps are based on partially masked inputs, the model explicitly learns to solve complex tasks from random sub-tasks, which benefits their reasoning abilities \citep{kim2025trainworstplanbest}.
\LTsAcr, in contrast, only receive training signal from the final decision and have to construct the sub-steps of the computation on their own.
This could lead to suboptimal utilization of the sequential computation or even to failure to use it at all.
Similarly, the human-understandable sub-tasks that \MDMsAcr have to solve make their training more interpretable and easier to control. 
On the other hand, the more information-rich intermediate states of \LTsAcr make them more efficient at storing and processing information, and the lack of sampling steps makes them more efficient at inference time.

\paragraph{Generalizations.}
By focusing on transformer-based \MDMsAcr, we can draw from the rich theory developed on the expressivity of transformers. 
However, \MDMsAcr do not have to be implemented by a transformer---they could, for example, be implemented by a state-space model.
Nevertheless, the parallelizable nature of \MDMsAcr suggests that any reasonable real-world implementation will include parallelizable components---for example, a model implementable by a constant-depth circuit, such as a $\TCZero$ circuit.\footnote{A similar modeling assumption is made by \citet{liu2025perfectdiffusionmathsftc0,liu2025serialscalinghypothesis} for latent diffusion LMs.}
The close connection between transformers (with logarithmically-growing precision and padding) and the class $\TCZero$ \citep{merrill-sabharwal-2023-parallelism,li2024chain} suggests that the results of \cref{sec:theoretical-results} will largely carry over to such implementations.
We conjecture that, regardless of whether \MDMsAcr are implemented by finite-precision transformers ($\ACZero$ circuits) or a more expressive $\TCZero$ circuit, $\MDMClass[\log^d\strlen, \polyFun{\strlen}]$ would remain in $\ACd$ (cf.~\cref{cor:mdms-with-constant-steps}) or a similar class like $\TCd$.
Moreover, we state the sequentiality bottleneck only for $\log \strlen$ decoding steps, since the expressivity of $\CoTClass[\log\strlen]$ is known to be limited.
However, we believe that a similar separation exists for polylogarithmically many decoding steps: While the expressivity of $\CoTClass[\log^d\strlen]$ has not been formalized yet, it likely does not capture all of $\ACd$, unlike $\MDMClass[\log^d\strlen, \polyFun{\strlen}]$ (cf.~\cref{cor:mdms-with-constant-steps}).

\paragraph{Discrepancies.}
We strive to compare different paradigms fairly by analyzing a specific implementation of \MDMsAcr---one based on a specific idealization of transformers.
Some inherent differences between the frameworks, however, remain.
One is the dichotomy between using causal masking for \CoTAcr and unmasked transformers for \MDMsAcr and \LTsAcr.
We focus on unmasked models to analyze the more natural and popular implementations rather than artificially constraining \MDMsAcr and \LTsAcr to causal masking.

\paragraph{Positional encodings.}
Another impactful decision is the nature of positional encodings (PEs).
While we assume relatively simple PEs standard in theoretical literature (in particular, logspace-computable, cf. \cref{app:transformers}), we allow them to come from an outside source not part of the transformer---they may thus carry information not computable by the model.
This highlights a crucial feature: The computational capability in the $\LUniform$ transformer family is deliberately split between the PEs and the transformer's computation mechanism.
Fixed-precision transformers are limited to a subset of $\ACZero$ and cannot compute division, modulo, or binary representations \citep{li2024chain}.
The PEs provide this missing information: Binary encodings of position/length and results of division/modulo operations (see \cref{app:positional-encodings}).
The transformer mechanism can then perform $\ACZero$ operations (addition, thresholding, multiplication by powers of two) on this pre-computed data.
Thus, in \cref{thm:regular-languages-in-mdm-efficient}, the transformer is looking up pre-computed (but still simple, log-space) information from the PEs rather than computing it.
When chaining decoding steps, modulo and division information is needed for correct indexing---making $\LUniform$ PEs a prerequisite for ``looping'' to increase computational power beyond fixed-depth models.
We emphasize that offloading computation to PEs is standard in the expressivity literature \citep{li2024chain,saunshi2025reasoninglatentthoughtspower,london2025pausetokensstrictlyincrease}, often encoding entire circuits to be simulated.
In contrast, our PEs are simpler: Functions of input position and length, not circuit-specific.

\paragraph{Relaxing assumptions.}
Our constructions rely on a planner that perfectly unmasks the input string and an unmasker that models the conditional distributions perfectly. 
These are idealizations of confidence-based planners and learned unmaskers. 
How a trained planner can learn to approximate this ideal behavior is an important question---one of learnability rather than pure expressivity. 
As is standard in expressivity analysis of language models, expressivity is a necessary first step before one can analyze how a model can \emph{learn} to represent such functions. 
While learnability of formal languages is an active area of research, we are not aware of any existing results that would characterize the learnability of functions required for planning (e.g., modulo and division) apart from the issue of the sensitivity of modular counting \citep{hahn-rofin-2024-sensitive}, which suggests that learning planning behaviour used in our constructions is difficult for transformers. 

\paragraph{Empirical validation.}
The main purpose of this work is largely theoretical. 
It provides a framework for understanding the expressivity of \MDMsAcr in relation to other popular paradigms and provides new insights into the tasks that \MDMsAcr are well-suited for.
In fact, this analysis is motivated by existing empirical results that show \MDMsAcr outperforming autoregressive models on tasks that allow for parallel reasoning \citep{kim2025trainworstplanbest,ye2025autoregressiondiscretediffusioncomplex}.
The lack of literature on training \MDMsAcr on formal languages, however, makes it difficult to draw concrete empirical implications of the presented results.
Thus, we see designing targeted experiments and benchmarks that empirically explore the theoretical benefits of \MDMsAcr as an important direction for future work.

\section{Conclusion} \label{sec:conclusion}
We describe the expressivity of masked diffusion LMs (\MDMsAcr) by connecting them to padded looped transformers (\LTsAcr) and chain-of-thought-augmented (\CoTAcr) transformers.
This reveals a close connection between \LTsAcr and \MDMsAcr, which leads to the equivalence of \MDMAcr with polylogarithmically many denoising steps to the class $\NC$ of parallelizable problems, with concrete implications around what problems can benefit from the parallelism afforded by \MDMsAcr.
We also show that \MDMsAcr can (somewhat inefficiently) simulate \CoTAcr transformers. 
We describe the sequentiality bottleneck and the strict expressivity gap between \MDMsAcr and \CoTAcr transformers with logarithmically many model evaluations.
This shows \MDMsAcr to be more suitable for highly-parallelizable problems, while \CoTAcr transformers are more suitable for inherently sequential ones.
Overall, our results provide insights into the strengths and weaknesses of \MDMsAcr and their suitability for different classes of problems.

\section*{Ethics Statement}
This work is theoretical and aims to describe the capabilities of masked diffusion models to better understand their strengths and limitations. 
We do not foresee any direct negative societal impacts.

\section*{Reproducibility Statement}
All our results are theoretical and thus reproducible from the provided proofs in \cref{app:theoretical-gadgets,app:proofs}.

\section*{The use of large language models} \label{app:llm-use}
We used AI-based tools (Gemini and GitHub Copilot) for brainstorming and writing assistance.
We used the tools in compliance with the ICLR 2026 policies.

\subsubsection*{Acknowledgments}
Anej Svete is supported by the ETH AI Center doctoral fellowship.

\bibliography{bibliography}
\bibliographystyle{iclr2026_conference}

\newpage

\appendix

\section{Related work on the expressivity of \MDMsAcr} \label{app:related-work}

Existing theoretical work on \MDMsAcr focuses on the limitations of the factorized backward process and its convergence properties, which can be linked to formal language generation and recognition.

\citet{feng2025theoreticalbenefitlimitationdiffusion} analyze the implications of the factorized backward process on the expressivity of \MDMsAcr.
They find that \MDMsAcr can approximate \ngram LM distributions arbitrarily well with constantly many sampling steps and linearly many sampling steps suffice to approximately sample from any regular language. 
They also provide a linear lower bound on the number of sampling steps required to capture the support of general regular languages---a consequence of the assumed factorization, which is incompatible with the sequential generation of regular languages.
Crucially, this differs from our \emph{recognition} setting in which the string to be processed is given and only its membership decision has to be generated.

This line of work is tightened by \citet{li2025convergencetheorydiffusionlanguage}, who build on work by \citet{chen2024convergenceanalysisdiscretediffusion} and analyze the approximation error (measured by the KL divergence) of the distribution generated by \MDMsAcr with respect to the number of sampling steps and the mutual information (dependence) between the symbols in different positions.
Intuitively, the larger the mutual information is, the larger the number of sampling steps has to be to capture the same distribution (equivalently, the fewer symbols per step can be generated).
They show that the KL divergence between the generated and the ground-truth distributions decays linearly with the number of sampling steps (with scaling that depends on the mutual information between the symbols in the strings), and show this decay to be optimal in general.
This generalizes the results on \ngram distributions and regular languages by \citet{feng2025theoreticalbenefitlimitationdiffusion}.

\citet{liu2025perfectdiffusionmathsftc0,liu2025serialscalinghypothesis} take a different perspective and analyze the expressivity of latent diffusion LMs.
They show that, under an analogous assumption to our \cref{assumption:uniform-unmasking,assumption:perfect-approximation}, latent diffusion LMs converge to the data distribution in constantly many steps, which means that the computational depth of such models is limited.
They show that breaking this assumption makes latent diffusion LMs Turing complete, analogous to our results in \cref{sec:mdms-and-cot}.

While these results provide useful insights into the limitations and affordances of \MDMsAcr, their asymptotic and approximate nature makes it difficult to draw concrete conclusions about the (reasoning) capabilities of \MDMsAcr in the sense of the work on transformers' expressivity.
Our work complements these results by providing a more fine-grained analysis of \MDMAcr capabilities based on their connection to \LTsAcr and \CoTAcr transformers, which allows us to leverage the existing theory developed on transformer expressivity and apply it directly to \MDMsAcr.

\section{Preliminaries} \label{app:preliminaries}

\subsection{Notation} \label{app:notation}

Let $\alphabet$ be an alphabet.
A \defn{language} $\lang$ is a subset of $\strings$.
A \defn{language recognizer} is a function $\recognizer\colon \strings \to \{\rejectSym, \acceptSym\}$, where $\rejectSym$ and $\acceptSym$ are designated reject and accept symbols. 
$\recognizer$'s language is $\langFun{\recognizer} \defeq \set{\str \in \strings \mid \recognizerFun{\str} = \acceptSym}$.
Two recognizers $\recognizer_1$ and $\recognizer_2$ are \defn{equivalent} if and only if $\langFun{\recognizer_1} = \langFun{\recognizer_2}$.

We denote the concatenation of two strings $\str_1, \str_2 \in \strings$ as $\str_1 \circ \str_2$ or simply $\str_1\str_2$.
We define the \defn{interleaving} of the vectors $\vx, \vy \in \R^\hiddDim$ as $\interleaveFun{\vx}{\vy} \in \R^{2\hiddDim}$ where 
\begin{equation}
   \interleaveFun{\vx}{\vy}_\dimidx \defeq 
   \begin{cases}
      \evx_{\sfrac{(\dimidx+1)}{2}} & \ifcondition \dimidx \text{ is odd}, \\
      \evy_{\sfrac{\dimidx}{2}} & \otherwisecondition \dimidx.
   \end{cases}
\end{equation}

We use $\onehot{\sym} \in \{0, 1\}^{\nsymbols}$ to denote the one-hot encoding of symbol $\sym \in \alphabet$.
We use $\bin_\numPrec(\stridx)$ to denote the binary encoding of natural number $\stridx$ using $\numPrec$ binary bits and $\sbin_\numPrec(\stridx)$ to denote the signed binary encoding $2\bin_\numPrec(\stridx)-\one_\numPrec$, where $\one_\numPrec$ is the $\hiddDim$-dimensional vector of all ones.
We will leave out $\numPrec$ when it is clear from the context.

For $\hiddDim \in \N$, we define $\softmax\colon\R^\hiddDim\to\R^\hiddDim$ as $\softmaxFun{\vx}{\dimidx} = \exp(\evx_\dimidx)/\sum_{\dimidx=1}^\hiddDim \exp(\evx_\dimidx)$ for $\vx\in\R^\hiddDim$ and $\dimidx\in [\hiddDim]$. 
We also use the shorthand $\posPart{x} \defeq \max\{x, 0\}$.
We denote with $\probOverFun{\sX}$ the set of all probability distributions over a set $\sX$.

\subsection{Circuit Complexity} \label{app:circuit-complexity}
Computational circuits are a model of parallel computation.
They have been widely used in the study of the expressivity of neural networks.
Circuits process binary input strings through a series of logical operations to produce binary outputs.\footnote{By representing symbols from any alphabet with binary encodings, circuits (or circuit functions) can be used to process strings over any finite alphabet. We focus on binary strings for simplicity.}
Formally, a \defn{boolean circuit} is a directed acyclic graph where source nodes represent the \strlen-bit input, and a single sink node represents the output. 
Non-source vertices are called \defn{gates} and are labeled with logical operations (e.g., \andGate, \orGate, \notGate). 
The \defn{size} of a circuit is the number of gates, and its \defn{depth} is the longest path from any input to the output.

A circuit computes a function $\circuit\colon \{0,1\}^\strlen \to \{\rejectSym, \acceptSym\}$ for some $\strlen \in \N$, where $\rejectSym$ and $\acceptSym$ are designated reject and accept symbols.
The value $\circuitFun{\str}$ for input string $\str \in \{0,1\}^\strlen$ is computed by evaluating the gates in topological order starting from the input bits.
We say that the circuit $\circuit$ \defn{accepts} a string $\str$ if $\circuitFun{\str} = \acceptSym$.

\defn{Circuit families} process input strings of variable length.
A circuit family is a sequence of circuits $\circuitFamily \defeq \{\circuit_{\strlen}\}_{\strlen \in \N}$ where $\circuit_\strlen$ processes inputs of length $\strlen$. 
A circuit family is said to recognize a language if for any given input string, the corresponding circuit outputs $\acceptSym$ if and only if the string is in the language.

A \defn{circuit complexity class} is a set of circuit families that satisfy certain constraints on size, depth, and the types of gates used.
This paper focuses on two common classes:
\begin{itemize}[noitemsep,topsep=0pt]
   \item \defn{$\ACd$}: Circuits with \notGate, \andGate, and \orGate gates that have unbounded fan-in and depth $\bigOFun{\log^d\strlen}$.
   \item \defn{$\TCd$}: The extension of $\ACd$ that adds \defn{threshold gates}, which output $\acceptSym$ if the sum of their inputs exceeds a given threshold. 
   It is known that $\ACZero \subsetneq \TCZero$ and $\ACd \subseteq \TCd$.
   For example, \textsc{Parity}, the language of binary strings with an even number of 1s, is in $\TCZero$ but not in $\ACZero$ \citep{Furst1984}.
   \item \defn{$\NCOne$}: This class consists of circuits that can be computed in parallel with a logarithmic depth, a polynomial number of gates, and constant fan-in. 
   It is known that $\TCZero \subseteq \NCOne$.
\end{itemize}

Without additional constraints, circuit families can recognize undecidable languages by having arbitrary, ``hard-coded'' solutions for each input length. 
To avoid this and ensure the model of computation is realistic, we can impose a \defn{uniformity} condition. 
A circuit family is \defn{uniform} if there exists a Turing machine that, given an input of $1^\strlen$, can generate a description of the circuit $\circuit_\strlen$.
In particular, a circuit class is \defn{$\LUniform$} if a Turing machine using $\bigOFun{\log \strlen}$ space can construct its description from the input $1^\strlen$.
This ensures the circuits for different input lengths are related by a systematic procedure.

\subsection{Finite-precision fixed-point arithmetic} \label{app:fixed-point-arithmetic}

We assume that the operations performed by our computational models rely on finite-precision fixed-point arithmetic.
This model is based on ones used by \citet{li2024chain,saunshi2025reasoninglatentthoughtspower,london2025pausetokensstrictlyincrease}. 

\begin{definition}[Fixed-Point Representation] 
   Let $\numPrec \in \N$ be the number of bits devoted to each of the integer and fractional parts. 
   We use $\F_\numPrec$ to denote the set
   \begin{equation}
      \F_\numPrec \defeq \{ x_{\pm} \cdot a \cdot 2^{-\numPrec} \mid x_{\pm} \in \{-1, 1\}, a \in \{0, 1, \ldots, 2^{2\numPrec}-1\} \}
   \end{equation}
\end{definition}
We define $\maxFNum \defeq \max \F_\numPrec = 2^\numPrec - 2^{-\numPrec}$.
All values exceeding $\maxFNum$ are considered out of range and are rounded to $\maxFNum$.
Note, however, that $\maxFNum$ does \emph{not} behave like infinity---it does not ``consume'' all subsequent operations.
For example, $\maxFNum - x \neq \maxFNum$ for some non-negative $x \in \F_\numPrec$ is a valid number.

To handle the results of arithmetic operations that may not be exactly representable in the fixed-point format, we define a standard for rounding.
\begin{definition}[Rounding]
   For any $x \in \R$ and any closed subset $\F$ of $\R$ containing 0, we define rounding $\roundFun{x, \F}$ as the closest number to $x$ in $\F$. 
   In case of a tie, the value with the smaller absolute value is chosen.
\end{definition}
We denote the rounding operation as $\roundOpFun{\cdot} \defeq \round(\cdot, \F_\numPrec)$. 
This operation is applied to vectors and matrices element-wise.
All binary operations are defined by first performing the ideal mathematical operation and then rounding the result to the nearest representable value in $\F_\numPrec$. 
Division by zero is considered an error condition resulting in an incorrect output.
We also note that $\sfrac{\maxFNum}{2} = 2^{\numPrec - 1} - 2^{-\numPrec}$.

For operations involving more than two numbers, rounding is applied iteratively.
\begin{definition}[Summation with Iterative Rounding]
   For $\numPrec, N \in \N$ and $\vx \in \R^N$, we define summation with iterative rounding to $\numPrec$ fractional bits as the function $\textsc{sum}_\numPrec\colon \bigcup_{N\in\N}(\F_{\numPrec})^{N}\rightarrow\F_{\numPrec}$, where for any $N\in\N^{+}$ and $\vx \in (\F_\numPrec)^N$:
   \begin{equation}
      \textsc{sum}_\numPrec(\vx) \defeq \roundOpFun{\dots\roundOpFun{\roundOpFun{\evx_1 + \evx_2} + \evx_3} + \dots + \evx_N}
   \end{equation}
\end{definition}
This iterative rounding process is not associative and the order of operations can affect the final result. 
Based on this, we can also define more complex operations such as the \textbf{fixed-point inner product} $\langle \vx, \vy \rangle_\numPrec \defeq \textsc{sum}_\numPrec(\vx \odot \vy)$, where $\odot$ denotes the element-wise product of two vectors, and \textbf{fixed-point matrix product} for matrices $\mA$ and $\mB$, where $(\mA \times_\numPrec \mB)_{i,j} \defeq \langle (\mA_{i,:})^\top, \mB_{:,j} \rangle_\numPrec$.

\subsection{Transformers} \label{app:transformers}

We consider both \defn{unmasked} \citep{devlin-etal-2019-bert} and \defn{causally masked} transformers \citep{radford2019language}.
Concretely, we work with fixed-point transformers whose underlying arithmetic operations are replaced with their fixed-point counterparts.
A transformer $\tf$ consists of four parts: 
\begin{enumerate}[label=(\arabic*),topsep=0pt,noitemsep]
   \item a \defn{symbol embedding} $\symbolembedding\colon \alphabet \to \F^{\hiddDim}$ of the form $\symbolembedding(\sym) \defeq \inMtx \onehot{\sym}$ for $\sym \in\alphabet$, where $\inMtx \in \F^{\hiddDim \times \nsymbols}$ and $\onehot{\sym} \in \F^{\nsymbols}$ is the one-hot encoding of $\sym$,
   \item a \defn{positional encoding} $\posencoding\colon \N \times \N \to \F^\hiddDim$,
   \item $\numlayers$ \defn{layers} $\tflayer^{(1)}, \ldots, \tflayer^{(\numlayers)}$, each of which consists of two sub-layers: A multi-head self-attention layer and a position-wise fully-connected feed-forward network $\mlp$, and
   \item an \defn{output layer} $\transoutput$ of the form $\transoutput(\hiddState) \defeq \softmax(\outMtx \hiddState)$ for $\hiddState\in\F^{\hiddDim}$, where $\outMtx \in \F^{\eosnsymbols \times \hiddDim}$.
\end{enumerate}
Each layer has its own parameters and is indexed by the layer name and the depth for attention and feedforward layers.
We use $\hiddDim$ to denote the \defn{width} of a transformer.
A transformer with layers $\tflayer^{(1)}, \ldots, \tflayer^{(\numlayers)}$ computes $\hiddState^{(\layeridx)}_\stridx \in \F^{\hiddDim}$ for $\layeridx \in \set{1, \ldots, \numlayers}$ and each position $\stridx \in \NTo{\strlen}$ in the input string $\str = \sym_1\cdots\sym_\strlen \in \strings$ as follows:
\begin{subequations}
   \begin{align}
      \hiddState^{(0)}_\stridx &\defeq \symbolembedding(\sym_\stridx) + \posencoding(\stridx, \strlen) \in \F^{\hiddDim} \text{ for } \stridx \in \NTo{\strlen} \\
      \hiddMtx^{(\layeridx)} &\defeq \begin{pmatrix}
         \hiddState^{(\layeridx)\top}_1 &
         \cdots &
         \hiddState^{(\layeridx)\top}_\strlen
      \end{pmatrix}^\top \in \F^{\strlen \times \hiddDim} \\
      \queryMtx^{(\layeridx)} &\defeq \hiddMtx^{(\layeridx)} \mW_Q^{(\layeridx)}, \quad 
      \keyMtx^{(\layeridx)} \defeq \hiddMtx^{(\layeridx)} \mW_K^{(\layeridx)}, \quad
      \valueMtx^{(\layeridx)} \defeq \hiddMtx^{(\layeridx)} \mW_V^{(\layeridx)} \quad \in \F^{\strlen \times \hiddDim} \\
      \mG^{(\layeridx)} &\defeq \softmax(\attnMaskFun{\queryMtx^{(\layeridx)} \keyMtx^{(\layeridx)\top}}) \valueMtx^{(\layeridx)} + \hiddMtx^{(\layeridx)} \in \F^{\strlen \times \hiddDim} \label{eq:softmax-attention-matrix} \\
      \hiddMtx^{(\layeridx + 1)} &\defeq \mG^{(\layeridx)} + \mlp(\mG^{(\layeridx)}) \in \F^{\strlen \times \hiddDim}
   \end{align} 
\end{subequations}
Here, $\attnMask\colon \F^{\strlen \times \strlen} \to (\F \cup \set{-\infty})^{\strlen \times \strlen}$ is the \defn{masking function}.\footnote{Similar to \citet{li2024chain,saunshi2025reasoninglatentthoughtspower}, we define masking with a function rather than an additive matrix since subtracting $\maxFNum$ from an arbitrary number in $\F$ does not necessarily result in $-\maxFNum$.}\textsuperscript{,}\footnote{Examples of masking functions include the identity function (unmasked attention) and the identity function on the upper-triangular portion of the matrix that replaces the lower-triangular part with $-\maxFNum$ (causally masked attention).}
We say that the $\layeridx\textsuperscript{th}$ layer $\tflayer^{(\layeridx)}$ computes the function $\tflayer^{(\layeridx)}\colon \F^{\strlen \times \hiddDim} \to \F^{\strlen \times \hiddDim}$, defined by the function $\tflayer^{(\layeridx)}\colon \hiddMtx^{(\layeridx - 1)} \mapsto \hiddMtx^{(\layeridx)}$ for $\layeridx \in \set{1, \ldots, \numlayers}$.
We also denote with $\tf$ the function $\tf\colon \strings \to \F^{\strlen \times \hiddDim}$, defined as $\tf\colon \str \mapsto \hiddMtx^{(\numlayers)}$.

We use the following definition of the multi-layer perceptron.
\begin{definition}[Multi-layer perceptron] \label{def:mlp}
   A \defn{multi-layer perceptron} (MLP) is a function $\mlp\colon \F^\hiddDim \to \F^{\hiddDim'}$ that can be expressed as a composition of affine transformations and the $\ReLU$ activation function:
   \begin{equation}
      \mlp(\vx) = \ReLU(\mW_2 (\ReLU(\mW_1 \vx + \vb_1)) + \vb_2),
   \end{equation}
   where $\mW_1 \in \F^{H \times \hiddDim}$, $\mW_2 \in \F^{\hiddDim' \times H}$, $\vb_1 \in \F^{H}$, $\vb_2 \in \F^{\hiddDim'}$ for $\hiddDim, \hiddDim', H \in \N$.
\end{definition}

\subsubsection{Scaling transformer size with input length} \label{app:scaling-transformers}
The exact expressivity of a transformer model depends on seemingly unimportant details \citep{jerad2025uniquehardattentiontale}. 
One such example is the interplay between the scaling of the \defn{numerical precision} $\numPrec$ of the values stored in the representations $\hiddState$ and the scaling of the \defn{width} $\hiddDim$.
To be able to uniquely identify the $\strlen$ positions in the input string, the ``volume'' of the embedding space, i.e., the number of possible distinct representations $\hiddState$, must be at least $\strlen$.
This implies that the product $\numPrec \cdot \hiddDim$ must \emph{scale} at least logarithmically with $\strlen$: $\numPrecFun{\strlen} \cdot \hiddDimFun{\strlen} = \bigOmegaFun{\log\strlen}$.
Existing work focuses on two modeling choices: 
\begin{enumerate}[label=(\arabic*)]
   
   \item \defn{log-precision} transformers where $\numPrecFun{\strlen} = \bigThetaFun{\log\strlen}$ with either constant width $\hiddDimFun{\strlen} = \bigThetaFun{1}$ or polynomial width $\hiddDimFun{\strlen} = \bigThetaFun{\poly(\strlen)}$ \citep{merrill-sabharwal-2023-parallelism,merrill2024the,li2024chain,london2025pausetokensstrictlyincrease,merrill2025littledepthgoeslong,merrill2025exactexpressivepowertransformers}; and
   
   \item \defn{constant-precision logarithmic-width} transformers where $\numPrec = \bigThetaFun{1}$ and $\hiddDimFun{\strlen} = \bigThetaFun{\log\strlen}$ \citep{li2024chain,saunshi2025reasoninglatentthoughtspower,london2025pausetokensstrictlyincrease}.
   
\end{enumerate}
Despite having the same volume of the representation space, the two models differ in their expressivity.
For example, constant-precision transformers are constrained to use the volume in a ``distributed'' manner across model dimensions without the ability to summarize the information into individual values or store pointers to arbitrary positions in them---both of these require precision growing with string length.
Such summarization is required in certain steps of the transformer architecture---for example, when computing the attention scores.
This limits the expressivity of constant-precision transformers compared to log-precision transformers: While log-precision transformers can compute certain $\TCZero$ functions, constant-precision transformers with polynomial width fall within $\ACZero$ \citep{li2024chain,london2025pausetokensstrictlyincrease}.
Similar separation results extend to popular variants of transformers:
\begin{enumerate*}[label=(\alph*)]
   \item Transformers with chain-of-thought reasoning (cf. \cref{app:cot}) with logarithmic precision can simulate $\TCZero$ circuits of size corresponding to the number of reasoning steps while constant-precision transformers with polynomial width are constrained to $\ACZero$ circuits \citep{li2024chain,merrill2024the}.
   \item Transformers with additional padding space (blank thinking tokens; cf. \cref{app:looped-padded-transformers}) with logarithmic precision can express exactly $\TCZero$ functions while constant-precision transformers with polynomial width are equivalent to $\ACZero$ circuits \citep{merrill2025exactexpressivepowertransformers,london2025pausetokensstrictlyincrease}.
\end{enumerate*}
However, the trend of coarse quantization while increasing the model size makes fixed-precision logarithmic-width transformers particularly appealing, which is why we focus on \emph{finite precision} transformers with \emph{logarithmic width}.

Analogously to circuit families, each string length $\strlen$ is processed by a separate transformer model.
To process all of $\strings$, we therefore define a \defn{transformer family} $\{\tf_\strlen\}$ as a sequence of transformers where each $\tf_\strlen$ processes strings of length $\strlen$.
Further, we again impose a uniformity condition on the family, which will relate the transformers for different input lengths.

\begin{definition}[Uniform transformer families; variant of \text{\citealp[][Def. 3.6]{london2025pausetokensstrictlyincrease}}] \label{def:uniform-transformers}
   Let $\XCls$ be a computational complexity class.
   A transformer family $\{\tf_\strlen\}$ is \defn{$\XUniform$} if there exist Turing machines $\tm_1$ and $\tm_2$ whose resource usage is constrained by the complexity class $\XCls$ such that: 
   \begin{enumerate}[label=(\arabic*),topsep=0pt,noitemsep]
      \item $\tm_1$ takes input $1^\strlen$ and outputs a description of $\tf_\strlen$, and
      \item $\tm_2$ takes input $(1^\strlen, \bin(\stridx))$ and outputs $\posencoding(\stridx, \strlen)$.
   \end{enumerate} 
\end{definition}

\cref{def:uniform-transformers} allows for size-dependent transformers while keeping them closely related (as the same Turing machines must construct them for all $\strlen$). 
It also facilitates natural connections with uniform circuit classes (cf.~\cref{app:circuit-complexity}) \citep{london2025pausetokensstrictlyincrease}.
All our results concern $\LUniform$ transformer families, in which case, the Turing machines in \cref{def:uniform-transformers} operate in logarithmic space.

\subsubsection{Transformer language models and symbol predictors} \label{app:language-models-and-next-symbol-predictors}
Transformers can implement (non-)autoregressive LMs and deterministic symbol predictors. 
A transformer LM computes the probability distribution over the $\strlen + 1\textsuperscript{st}$ symbol given the length-$\strlen$ string $\str$ as
\begin{equation}
   \pLMFun{\textcolor{ETHPurple}{\eossym} \mid \str} \defeq \transoutput(\hiddState^{(\numlayers)}_{|\str|})_{\textcolor{ETHPurple}{\eossym}} \qquad \text{for } \textcolor{ETHPurple}{\eossym} \in \eosalphabet.
\end{equation}
To define \defn{infilling} probabilities, let $\str \in \kleene{\maskAlphabet}$ be a possibly (partially) masked string of length $\strlen$ and let $\stridx \in \NTo{\strlen}$ such that $\sym_\stridx = \maskSym$.
We then define the probability distribution over the symbol at the masked position $\stridx$ as
\begin{equation}
   \pLMInfFun{\textcolor{ETHPurple}{\sym} \mid \str} \defeq \transoutput(\hiddState^{(\numlayers)}_{\stridx})_{\textcolor{ETHPurple}{\sym}} \qquad \text{for } \textcolor{ETHPurple}{\sym} \in \alphabet.
\end{equation}
Note that the next-symbol probabilities are computed based on the contextual representation of the previous symbol while the infilling probabilities are defined based on the contextual representation at the masked position.
That is, for autoregressive next-symbol prediction, the transformer uses the hidden state at the last (previous) position, whereas for infilling, it uses the hidden state at the position to be filled. 

To define a deterministic transformer-based symbol predictor, we define a decoding step.
\begin{definition}[Decoding step] \label{def:decoding-step}
   For any $\strlen \in \N$, let $\residStream \in \R^{\strlen \times \hiddDim}$ .
   The \defn{decoding step} $\decodeStep\colon \R^{\strlen \times \hiddDim} \to \eosalphabet^{\strlen}$ is defined as
   \begin{equation}
      \decodeStep(\residStream)_\stridx = \argmax_{\sym \in \eosalphabet} \transoutput(\residStream)_{\stridx, :}
   \end{equation}
   where the output function $\transoutput$ is applied to $\residStream$ row-wise.
   This can be used either to deterministically infill the masked positions in the string or to deterministically predict the next symbol based on the final representation.
   For next-symbol prediction, we also define the shorthand $\decoder \colon \strings \to \eosalphabet$ as
   \begin{equation}
      \decoderFun{\str} \defeq \decodeStep(\tf(\str))_{\strlen}.
   \end{equation}
\end{definition}

\subsubsection{Language encoders and model equivalence} \label{app:model-equivalence}
We are interested in the expressivity of neural networks as language recognizers and LMs, which at a high level, describes their behavior on input strings.
This behavior---either the prediction of language membership or the computation of string probabilities---is completely determined by the \defn{contextual representations} of the input strings produced by the neural network.
We abstract the computation of string representations using a \defn{language encoder}---a length-preserving function $\langEnc\colon \strings \to \kleene{(\R^{\hiddDim})}$ for some $\hiddDim \in \N$ \citep{10.5555/3737916.3740249,cotterell2024formalaspectslanguagemodeling}.
We regard the output $\langEnc(\str)$ of a language encoder for a string $\str$ as a $|\str| \times \hiddDim$ matrix, where each row corresponds to the contextual representation of the symbol at the corresponding position.
In the context of transformers, the language encoder is the function $\langEncFun{\str} \defeq \hiddMtx^{(\numlayers)}$ for $\str \in \strings$.

All aspects of the model's behavior that we might be interested in can be described in terms of the contextual representations---the (logits of the) next-symbol or infilling probabilities are determined by a linear transformation of the contextual representations, and the membership test is determined by a linear classifier based on the contextual representations of the final symbol in the string.
Studying the expressivity of a model thus reduces to determining what types of contextual representations can be produced by the model.
If we can show that two language encoders $\langEnc_1, \langEnc_2 \colon \strings \to \kleene{(\R^\hiddDim)}$ compute the same contextual representations for all input strings, we say that $\langEnc_1$ and $\langEnc_2$ are \defn{equivalent}.
Moreover, we say that two sets of models are equivalent if each model in one set is equivalent to at least one model in the other set.

\paragraph{Model equivalence and variable-length outputs.}
Sometimes, we will compare the expressivity of models that produce outputs of different lengths, for example when comparing the expressivity of padded and non-padded models, or comparing \CoTAcr-augmented transformers with non-augmented ones.
This makes direct comparison of contextual representations more difficult.
Whenever this is the case, we will explicitly state how the outputs of different lengths are aligned.
For example, when comparing a model that produces outputs of length $\strlen$ with a model that produces outputs of length $K \strlen$ for some constant $K \in \N$, we may assume that the first $(K - 1) \strlen$ positions of the longer output are used for intermediate computations and only the last $\strlen$ positions are used to produce the final output.
Thus, we will only compare the last $\strlen$ positions of the longer output with the $\strlen$-length output of the shorter model and base model equivalence on that.

\subsubsection{Sampling from a (transformer) language model} \label{app:sampling-from-language-models}
The next-symbol and infilling probabilities from \cref{app:language-models-and-next-symbol-predictors} can be used to sample from the LM by sampling from $\pLMFun{\cdot \mid \str}$ or $\pLMInfFun{\cdot \mid \str}$, respectively.\footnote{In practice, the raw probabilities are often post-processed with temperature scaling or sampling adaptors \citep{amini-etal-2023-generating}. The framework described here can easily be adapted to those settings.}
There are many possible ways to implement the sampling.
In this paper, we assume that it is performed using the Gumbel-max trick \citep{pmlr-v97-oberst19a}, which both provides a convenient and fast implementation as well as interpretable traces of the sampling procedure \citep{chatzi2024counterfactual,ravfogel2025gumbel,benz2025evaluation}.
\begin{definition}[Gumbel-max sampling] \label{def:gumbel-max-sampling}
   Let $\vl \defeq \outMtx \hiddState$ for $\hiddState\in\R^{\hiddDim}$ and $\outMtx \in \R^{\eosnsymbols \times \hiddDim}$ be the logits of a probability distribution over $\eosalphabet$.
   The \defn{Gumbel-max sampling} is defined as
   \begin{equation} \label{eq:gumbel-max-sampling}
      \sym = \argmax_{\sym' \in \eosalphabet} (\evp_{\sym'} + g_{\sym'}),
   \end{equation}
   where $g_{\sym'}$ are i.i.d. samples from the Gumbel distribution with the cumulative distribution function $F(x) = \exp(-\exp(-x))$ for $x \in \R$.
\end{definition}
It is well known that \cref{eq:gumbel-max-sampling} results in samples from $\softmaxFun{\vl}{}$.
The Gumbel-max sampling can be used to either sample the next symbol $\sym_{\strlen + 1}$ from the next-symbol distribution $\pLMFun{\cdot \mid \str}$ or to sample the masked symbol $\sym_\stridx$ from the infilling distribution $\pLMInfFun{\cdot \mid \str}$.

We rely on the Gumbel-max sampling because it conveniently decouples sampling from the model's representations.
In particular, provided that a model is able to implement the $\argmax$ operation and is able to receive the Gumbel noise as an input, it can sample from the model's distribution without any additional operations.
This will provide a convenient way to precisely link different modeling frameworks in a unified manner.

\subsection{Chain-of-thought reasoning} \label{app:cot}

At a high level, chain-of-thought (\CoTAcr) transformers process a string by generating intermediate reasoning steps. 
These steps can be seen as a generated string itself or as an intermediate thinking process that is used to generate the final output.
\begin{definition}[\CoTAcr Generation]
   A causally masked transformer implementing the next-symbol distribution $\pLM \colon \strings \to \probOverFun{\eosalphabet}$ \defn{generates} strings $\unmsym_1 \cdots \unmsym_\timesteps$, given $\str \in \strings$, as 
   \begin{equation}
      \unmsym_t \sim \pLM(\cdot \mid \str \circ \unmsym_1 \cdots \unmsym_{t-1}) 
   \end{equation}
   where $\unmsym_\timesteps = \eos$ and $\unmsym_t \neq \eos$ for $t < \timesteps$.
\end{definition}

While some existing work analyzes the distributions induced by \CoTAcr transformers \citep{nowak-etal-2024-representational,xu2025cotloopformalcomparison}, much of the existing literature \citep{JMLR:v22:20-302,feng2023towards,merrill2024the,li2024chain,saunshi2025reasoninglatentthoughtspower} focuses on modeling string \defn{acceptance} by \CoTAcr transformers by determinizing $\pLM$.
\begin{definition}[\CoTAcr Acceptance]
   A causally masked transformer implementing the next-symbol predictor $\decoder\colon \strings \to \acceptAlphabet$ generates the sequence of reasoning steps 
   \begin{equation}
      \unmsym_t \defeq \decoder(\str \circ \unmsym_1 \cdots \unmsym_{t-1}) 
   \end{equation}
   for $\str \in \strings$, $t \in \NTo{\timesteps}$, and a pre-determined $\timesteps \in \N$.
   $\decoder$ \defn{accepts} a string $\str$ in $\timesteps$ steps if $\unmsym_\timesteps = \acceptSym$ and rejects it if $\unmsym_\timesteps = \rejectSym$.
\end{definition}

To facilitate a more convenient connection to \MDMsAcr, we introduce a \emph{parallel} chain-of-thought process that can be seen as a generalization of \CoTAcr transformers, one that generates multiple symbols at once.
\begin{definition}[Parallel chain of thought (\pCoTAcr)]
   Let $\predictor \colon \kleene{\maskAlphabet} \to \maskAcceptAlphabet$ be a causally masked transformer symbol predictor and $\timesteps, \padlen' \in \N$.
   A \defn{parallel chain-of-thought transformer} $\decoderpar\colon \strings \to \kleene{\maskAlphabet}$ processes a string $\str$ from $\strings$ for $t \in \NTo{\timesteps}$ as follows:
   \begin{equation}
      \thinkstr^{(t)}_\stridx \sim \predictor(\str \circ \thinkstr^{(1)} \circ \cdots \circ \thinkstr^{(t - 1)} \circ \underbrace{\maskSym \cdots \maskSym}_{\padlen'})_{\strlen + (t - 1) \padlen' + \stridx}  \text{\quad for } \stridx \in \NTo{\padlen'}.
   \end{equation}
   Whenever $\timesteps$ and $\padlen$ are clear from the context, we write $\decoderpar(\str) \defeq \thinkstr^{(0)} \circ \cdots \circ \thinkstr^{(\timesteps)}$.
\end{definition}

A \pCoTAcr transformer generates strings in $\strings$ by running the \pCoTAcr process on the empty string $\eps$ and outputting $\decoderpar(\eps)$.
\begin{definition}[Acceptance by a \pCoTAcr Transformer]
   We say that a \pCoTAcr transformer \defn{accepts} a string $\str$ if there exist $\timesteps, \padlen' \in \N$ such that it holds for $\thinkstr^{(1)} \circ \cdots \circ \thinkstr^{(\timesteps)} = \decoderpar(\str)$ that $\thinkstr^{(\timesteps)}_{\padlen'} = \acceptSym$. 
   $\decoderpar$ rejects $\str$ if $\thinkstr^{(\timesteps)}_{\padlen'} = \rejectSym$. 
\end{definition}
We denote the class of languages recognizable by a \pCoTAcr transformers that generate $\padlen'$ symbols at a time for $\timesteps$ steps as $\pCoTClass[\timesteps, \padlen]$, where $\padlen \defeq \padlen' \timesteps$ is the total number of generated symbols.

The following relationships between \CoTAcr and \pCoTAcr transformers are clear.
\begin{restatable}[\CoTAcr and \pCoTAcr]{proposition}{CoTAsPCoTProp} \label{prop:cot-as-pcot}
   We have $\pCoTClass[\timesteps, \timesteps] = \CoTClass[\timesteps]$ and $\pCoTClass[\timesteps, \padlen] \subseteq \CoTClass[\padlen]$.
\end{restatable}

\subsection{Padded looped transformers} \label{app:looped-padded-transformers}

Looped (or universal) transformers use a fixed block of transformer layers that is applied repeatedly to the input string \citep{dehghani2019universaltransformers}.
This increases the depth of the model, enabling more complex reasoning by applying layers multiple times, and does not increase the model size, as the same block is reused for each iteration, thus reducing the memory footprint and computational cost \citep{bae2025mixtureofrecursionslearningdynamicrecursive}.
We define looped transformers as follows.
\begin{definition}[Looped transformer (\LTAcr)] \label{def:looped-transformer}
   Let $\numlayers, \timesteps \in \N$ and let $1 \leq \layeridx_1 \leq \layeridx_2 \leq \numlayers$.
   Given a depth-$\numlayers$ transformer, a \defn{looped transformer} (\LTAcr) computes symbol contextual representations $\hiddMtx$ by
   \begin{enumerate}[label=(\arabic*.)]
      \item Computing the initial hidden states $\hiddMtx^{(0)}$ for the input string $\str = \sym_1\cdots\sym_\strlen$ and computing $\hiddMtx^{(\layeridx_1)}$ as with the first $\layeridx_1$ layers of the transformer
      \item Applying the transformer layers $\layeridx_1 + 1, \ldots, \layeridx_2$ $\timesteps$ times to the hidden states $\hiddMtx^{(\layeridx_1)}$ to obtain $\hiddMtx^{(\layeridx_1 + \timesteps(\layeridx_2 - \layeridx_1))}$.
      \item Applying the transformer layers $\layeridx_2 + 1, \ldots, \numlayers$ to the hidden states $\hiddMtx^{(\layeridx_1 + \timesteps(\layeridx_2 - \layeridx_1))}$ to obtain the final representations $\hiddMtx$ that are passed to the output layer.
   \end{enumerate} 
\end{definition}
The representations $\hiddMtx$ can then be used in the same way as described in \cref{app:language-models-and-next-symbol-predictors}.

The dynamic computational depth of \LTsAcr endows them with the ability to perform more complex reasoning tasks by iteratively refining their hidden states over multiple timesteps.
Importantly, these reasoning steps include both sequential and parallel processing of the input symbols, allowing for both parallel efficiency and depth in the reasoning process.

Padded transformers additionally pad the input string with padding (pause) symbols.
\begin{definition}[Padded Transformer] \label{def:padded-transformer}
   Given $\padlen \in \N$, a \defn{padded transformer} $\tf$ transformer computes the contextual representations $\residStream$ of a string $\str \in \strings$ by processing the padded input $\str \circ \underbrace{\padSym \cdots \padSym}_{\padlen}$ (possibly by looping, cf. \cref{def:looped-transformer}), where $\padSym \notin \alphabet$ is a designated padding symbol.
\end{definition}
Instead of being restricted to the contextual representations of the $\strlen$ input symbols, a padded transformer can determine string membership or symbol probabilities based on the contextual representations of the $\padlen$ additional padded symbols as well.
This additional space can be used to perform more operations and is analogous to increasing the circuit width in circuit complexity.

Padding and looping together increase the expressivity of transformers.
\begin{restatable}[Expressivity of padded looped transformers]{remark}{LoopedTransformersCharacterization} \label{thm:looped-transformers-characterization}
   The following characterizations of padded looped transformers are known:
   \begin{enumerate}[label=(\arabic*)]
      \item $\text{Regular languages} \subseteq \LTClass[\log\strlen, 0]$ \citep[][Thm. 5.1]{saunshi2025reasoninglatentthoughtspower}, 
      \item $\CoTClass[\timesteps] \preceq \LTClass[\timesteps, 0]$, where the width of the \LTAcr scales linearly with $\timesteps$ \citep[][Thm. 5.4]{saunshi2025reasoninglatentthoughtspower}, 
      \item $\LTClass[\log\strlen, \polyFun{\strlen}] = \ACZero$ \citep[][Thm. 4.1]{london2025pausetokensstrictlyincrease}, 
      \item $\LTClass[\log\strlen, \polyFun{\strlen}] = \TCZero$, where, in contrast to our model, the precision $\numPrec$ \emph{scales} logarithmically with string length $\strlen$ (\citet[][Thm. 4.5]{london2025pausetokensstrictlyincrease}, \citet[][Thm. 1]{merrill2025exactexpressivepowertransformers}),
      \item $\LTClass[\log^d{\strlen}, \polyFun{\strlen}] = \TCd$, where the precision $\numPrec$ \emph{scales} logarithmically with string length $\strlen$ \citep[][Thm. 3]{merrill2025exactexpressivepowertransformers}. 
   \end{enumerate}
\end{restatable}

\paragraph{Stochastic padded looped transformers.}
The looping mechanism naturally accommodates the unmasking steps of \MDMsAcr.
However, unlike \LTAcr transformers, the \MDMAcr unmasking steps can be \emph{stochastic}---the predictor \emph{samples} the unmasked symbol from the infilling probability distribution defined by the transformer.
To offer a more suitable analogue to \MDMsAcr, we introduce \emph{stochastic} \LTsAcr, which receive Gumbel-distributed noise as additional input to each loop, mimicking the sampling process of \MDMsAcr.
\begin{definition}[Stochastic padded looped transformer] \label{def:stochastic-looped-transformer}
   A \defn{stochastic padded looped transformer} is a padded looped transformer in which $\hiddMtx^{(\layeridx_1 + t(\layeridx_2 - \layeridx_1))}$ is augmented by a matrix of Gumbel-distributed noise variables at each time step $t \in \NTo{\timesteps}$.
\end{definition}
The fact that the MLPs in a transformer layer can implement the $\argmax$ operation used for Gumbel sampling (cf. \cref{lem:mlp-for-argmax}) allows \LTsAcr to implement the planner and predictor of an \MDMAcr as we detail later (cf. \cref{thm:lts-can-simulate-mdms}).

Including stochasticity in \LTsAcr is required for a natural connection to stochastic models such as \MDMsAcr and \CoTAcr transformers.
While this departs from the standard definitions of \LTsAcr, it is a natural extension that allows us to capture the stochastic nature of \MDMsAcr while retaining the looping structure.
Since the Gumbel noise is assumed to come from an external source in the sampling procedure of \MDMsAcr and \CoTAcr transformers, we believe adding it as input to \LTsAcr is natural.

\subsection{Our idealization of masked diffusion models} \label{app:idealization}

In the following, $\maskAlphabet \defeq \alphabet \cup \{\maskSym\}$ where $\maskSym \notin \alphabet$ is the mask symbol and $\acceptAlphabet \defeq \alphabet \cup \{\rejectSym, \acceptSym\}$, where $\rejectSym$ and $\acceptSym$ are the reject and accept symbols.

\begin{definition}[Planner]
   A \defn{planner} is a length-preserving function $\planner\colon \kleene{\maskAlphabet} \to \probOver(\{\rejectSym, \acceptSym\})$.
   We distinguish two cases:
   \begin{enumerate}[label=\textbf{Case \arabic*:},leftmargin=*]
      \item An \defn{unrestricted planner} is a planner that can choose to resample any symbol.
      \item A \defn{mask dominated} planner is one where, if $\sym_\stridx \neq \maskSym$, then $\mleft(\planner(\str)_\stridx\mright)_{\rejectSym} = 1$, i.e., $\planner$ never tries to resample already unmasked symbols.
   \end{enumerate} 
   A \defn{deterministic} planner is a length-preserving function $\planner\colon \kleene{\maskAlphabet} \to \kleene{\{\rejectSym, \acceptSym\}}$.
\end{definition}

\begin{definition}[Symbol Predictor]
   A \defn{symbol predictor} is a length-preserving function $\predictor\colon \kleene{\maskAlphabet} \to \kleene{\probOver(\acceptAlphabet)}$.
   A \defn{deterministic} symbol predictor is a length-preserving function $\predictor\colon \kleene{\maskAlphabet} \to \kleene{\acceptAlphabet}$.
\end{definition}

\begin{definition}[Masked Diffusion Model]
   Given a planner $\planner$ and a symbol predictor $\predictor$, we call $\mdmModel = (\planner, \predictor)$ a \defn{masked diffusion model} (\MDMAcr).
\end{definition}

An \MDMAcr with a mask dominated planner corresponds closely to standard \MDMsAcr, where at each step, a subset of masked positions is selected for unmasking and then filled in.
For example, setting the planner to implement uniformly random unmasking recovers the true reverse process of an \MDMAcr.
We, however, additionally allow for unrestricted planners, which can also choose to \emph{resample} already unmasked positions and we take this to be our default setting---whenever we refer to an \MDMAcr, we mean one with an unrestricted planner.
While this departs from standard \MDMAcr formulations, it allows us to sidestep the limitations of the inability of the \MDMAcr to ``change'' its decisions and thus to correct earlier mistakes.
This has recently been identified as a key limitation of \MDMsAcr and is the focus of much recent work on improving \MDMsAcr \citep[][\textit{inter alia}]{rutte2025generalized,song2025seeddiffusionlargescalediffusion}.

\MDMsAcr generate strings by iteratively unmasking and filling in symbols over a series of discrete denoising steps.
We call this the \defn{unmasking process}.
\begin{definition}[Unmasking process]
   Let $\mdmModel$ be an \MDMAcr and $\timesteps, \padlen \in \N$. 
   Given a string $\str \in \strings$, $\mdmModel$ generates the string  $\mdmModel(\str) \defeq \unmstr^{(\timesteps)}$ as follows for $t \in \NTo{\timesteps}$:
   \begin{subequations}
      \begin{align}
         \unmstr^{(0)} &= \underbrace{\maskSym \cdots \maskSym}_{\padlen} \\
         \stru^{(t)} &\sim \planner(\str \circ \unmstr^{(t - 1)}) \\
         \unmstr^{(t)}_\stridx &=
         \begin{cases}
            \unmsym^{(t - 1)}_\stridx & \ifcondition \stru^{(t)}_\stridx = 0 \\
            \unmsym_\stridx \sim \predictor(\str \circ \unmstr^{(t - 1)})_{\strlen + \stridx} & \otherwisecondition
         \end{cases} \text{\quad for } \stridx \in \NTo{\padlen}
      \end{align}
   \end{subequations}
\end{definition}
By running the \MDMAcr on the empty string $\eps$ and outputting $\mdmModel(\eps)$ (constraining the generated symbols to $\alphabet \subsetneq \acceptAlphabet$), an \MDMAcr generates strings in $\strings$ and thus defines an LM.

We can also use the \MDMAcr to define a membership test for languages by using the unmasking process to emit reasoning (intermediate) computations and looking at the final symbol of the generated string.
\begin{definition}[Acceptance by an \MDMAcr]
   We say that the diffusion model $\mdmModel$ with a deterministic $\planner$ and $\predictor$ \defn{accepts} a string $\str\in \strings$ if there exist $\timesteps, \padlen \in \N$ such that it holds for $\unmstr^{(\timesteps)} = \mdmModel(\str)$ that $\unmsym^{(\timesteps)}_\padlen = \acceptSym$. 
   $\mdmModel$ rejects $\str$ if $\unmsym^{(\timesteps)}_\padlen = \rejectSym$. 
\end{definition}

\paragraph{Transformer \MDMsAcr.}
We focus on \MDMsAcr where both the planner and predictor are implemented by transformers.
\begin{definition}[Transformer planner and predictor]
   A \defn{transformer planner} is a planner $\kleene{\maskAlphabet} \to \probOver(\{\rejectSym, \acceptSym\})$ where the logits over $\{\rejectSym, \acceptSym\}$ are computed by a transformer.
   
   A \defn{transformer predictor} is a predictor $\predictor\colon \kleene{\maskAlphabet} \to \kleene{\acceptAlphabet}$ where the logits over $\acceptAlphabet$ at each position are computed by a transformer.
\end{definition}

\section{A discussion of the theoretical model} \label{app:theoretical-model} 

Our formalization of \MDMsAcr, centered around the planner and predictor, is motivated by the need to analyze their expressivity.
While this definition intentionally departs from the exact analytical reverse process, it remains well-grounded in both empirical practice and theoretical considerations.
In the following, we justify our modeling choices and clarify their connections to practical implementations.

\subsection{The planner as a principled departure from uniform unmasking} \label{app:planner-centrality}

At an intuitive level, our formalism distills the essential structure of practical \MDMsAcr: A process that iteratively unmasks and fills in symbols.
In our idealization, this process is governed by a model that chooses which symbols to unmask and when.
This is motivated by practical considerations---although the analytical reverse process of the \MDMAcr forward masking process would unmask symbols uniformly at random, this is not how practical \MDMsAcr operate, since they empirically benefit from strategic, context-aware unmasking \citep{ghazvininejad-etal-2019-mask,peng2025pathplanningmaskeddiffusion,zheng2024reparameterizeddiscretediffusionmodel,liu2025thinkgeneratediscretediffusion,kim2025trainworstplanbest,benhamu2025acceleratedsamplingmaskeddiffusion}.
Forcing a trained \MDMAcr to follow a uniform unmasking schedule is thus suboptimal---it prevents the model from decoding in an order that makes tasks easier to solve \citep{kim2025trainworstplanbest} and can lead to uncontrolled error propagation \citep{ghazvininejad-etal-2019-mask}.
The dedicated planner captures many empirically successful planned unmasking strategies that consistently outperform random unmasking such as:
\begin{itemize}
   \item \textbf{Confidence-based unmasking} that unmask based on the model's prediction confidence.
   \item \textbf{Difficulty-based scheduling} that masks informative symbols longer so they are generated last with maximum context \citep{he-etal-2023-diffusionbert}.
   \item \textbf{Structured schedules} such as blockwise-autoregressive unmasking \citep{nie2025largelanguagediffusionmodels}.
   \item \textbf{Learned planners} that use a separate trained model to guide unmasking decisions \citep{peng2025pathplanningmaskeddiffusion}.
   \item \textbf{Confidence thresholding} that unmasks symbols whose confidence lies above a theory-suggested threshold \citep{wu2025fastdllmtrainingfreeaccelerationdiffusion}.
\end{itemize}

Theoretically, uniform unmasking is only optimal if the symbol predictor is perfect and unfactorized, which is not the case in practice \citep{peng2025pathplanningmaskeddiffusion}.
Furthermore, under a factorized backward process, uniform unmasking is computationally inefficient even with a perfect predictor.
It requires at least a linear number of denoising steps in the string length to capture dependencies in even simple formal languages, and the KL divergence to the true data distribution converges slowly \citep{feng2025theoreticalbenefitlimitationdiffusion,li2025convergencetheorydiffusionlanguage}.
This lower bound negates any potential speed benefit over autoregressive models.
A planner is therefore a necessary component for efficient, accurate generation, as it provides an educated choice of which symbols to unmask and generate in parallel, mitigating the weakness of the independence assumption.
Frameworks such as that by \citet{zheng2024reparameterizeddiscretediffusionmodel} and augmented \MDMAcr evidence lower bounds that incorporate explicit planner terms \citep{peng2025pathplanningmaskeddiffusion,liu2025thinkgeneratediscretediffusion} demonstrate that a planner can be a principled, optimizable component of the generative process.

To further justify our choice of modeling the planner as a separate component, we show that any \MDMAcr that can be implemented as a combination of a planner and a predictor can in fact be implemented by a single transformer that unmasks symbols based on their confidence---$\topk$ decoding.
This means that all our results apply to the popular model of unmasking symbols based on their confidence \citep{ghazvininejad-etal-2019-mask,peng2025pathplanningmaskeddiffusion,zheng2024reparameterizeddiscretediffusionmodel,liu2025thinkgeneratediscretediffusion,kim2025trainworstplanbest,benhamu2025acceleratedsamplingmaskeddiffusion}.
\begin{definition}[$\topk$ unmasking]
   Let $\tf$ be a transformer and $k \in \N$. 
   The \defn{$\topk$ unmasking process} of $\tf$ is defined by the planner of the form 
   \begin{equation}
      \planner(\unmstr)_\stridx = 
      \begin{cases}
         \acceptSym & \ifcondition \stridx \in \topkFun{\tf(\unmstr)} \\
         \rejectSym & \otherwisecondition
      \end{cases}
      \text{ for } \stridx \in \NTo{\strlen},
   \end{equation}
   where $\unmstr \in \kleene{\maskAlphabet}$ with $\strlen = |\unmstr|$ and $\topkFun{\tf(\unmstr)}$ selects the $k$ positions in $\NTo{\strlen}$ with the largest maximal logits in $\tf(\unmstr)$.
   The predictor is defined as
   \begin{equation}
      \predictor(\unmstr)_\stridx = \pLMInfFun{\sym \mid \unmstr}
      \text{ for } \stridx \in \NTo{\strlen}.
   \end{equation}
   where $\pLMInfFun{\sym \mid \unmstr}$ is the infilling probability distribution of the $\stridx\textsuperscript{th}$ token defined by $\tf$.
\end{definition}
\begin{restatable}[Combining the planner and predictor]{theorem}{SingleModelEquivalenceThm} \label{thm:single-model-equivalence}
   Let $\mdmModel = (\planner, \predictor)$ be a masked diffusion model with a planner $\planner$ and a symbol predictor $\predictor$.
   Then, there exists a transformer $\tf$ such that it holds for any $\unmstr \in \kleene{\maskAlphabet}$ with $\strlen = |\unmstr|$ and $\stridx \in \NTo{\strlen}$ that
   \begin{subequations}
      \begin{align}
         \stridx \in \topkFun{\tf(\unmstr)} &\iff \planner(\unmstr)_\stridx = \acceptSym \\
         \pLMInfFun{\sym \mid \unmstr} &= \predictor(\unmstr)_{\stridx}
      \end{align}
   \end{subequations}
   where $\pLMInfFun{\sym \mid \unmstr}$ denotes the infilling probability distribution of the $\stridx\textsuperscript{th}$ token defined by $\tf$.
\end{restatable}
\begin{proof}
   The construction of $\tf$ combines the planner $\planner$ and predictor $\predictor$ into a single model that predicts the next symbol based on the planner's decision and the predictor's distribution.
   In particular, $\tf$ runs $\planner$ and $\predictor$ in parallel. 
   The planner's output decision $\planner(\unmstr)_\stridx$ of $\acceptSym$ or $\rejectSym$ can be made based on implementing the $\argmax$ function with an MLP (cf. \cref{lem:mlp-for-argmax}) if the planner is deterministic or by inserting the noise from the Gumbel-max sampling (cf. \cref{def:gumbel-max-sampling}) if the planner is probabilistic.
   $\tf$ can then use this information to  down-weight the logits of that are chosen not to be unmasked by the predictor by subtracting $\maxFNum$ from the logits of the symbols that are not chosen by the planner.
   This ensures that only the symbols that are chosen by the planner can be predicted by the predictor.
   The predictor $\predictor$ can in parallel compute its residual stream and the accompanying logits before combining them with the planner's decisions. 
   The subtraction of the $\maxFNum$ from the logits of the symbols that are not chosen by the planner ensures that $\topk$ will only select the symbols that are chosen by the planner and the simulation of the predictor ensures that the infilling distributions match.
\end{proof}

\subsection{Editing and non-editing \MDMsAcr.} \label{app:editing-mdms}

The choice to allow the \MDMAcr to resample already unmasked symbols also departs from the analytical reverse process of \MDMsAcr, which only unmask masked symbols.
This is, however, a principled choice that allows the \MDMAcr to correct earlier mistakes and is supported by empirical evidence that resampling already unmasked symbols can improve generation quality \citep{rutte2025generalized}.
In practice, this choice enables more concise connections to \LTsAcr and \pCoTAcr transformers.
Moreover, all results in this work can easily be adapted to \MDMsAcr with mask dominated planners that do not resample already unmasked symbols.
This is justified by the following theorem that shows that any \MDMAcr with an unrestricted planner can be simulated by an \MDMAcr with a mask dominated planner by increasing the output space by a factor of $\timesteps$ to account for the inability to resample.
In the following, we refer to \MDMsAcr with mask dominated planners as \emph{simple} \MDMsAcr (\sMDMsAcr).
\begin{restatable}[\sMDMsAcr can simulate \eMDMsAcr]{theorem}{MDMsCanSimulateFPLTsThm} \label{thm:mdms-can-simulate-emdms}
   \begin{equation}
      \MDMClassEdit[\timesteps, \padlen] \subseteq \MDMClassSimple[\timesteps, \timesteps \padlen].
   \end{equation}
\end{restatable}
\begin{proof}
   At a high level, the \sMDMAcr's planner selects the decoding space by selecting the next $\padlen$ positions to unmask, and the predictor
   \begin{enumerate*}[label=\textit{(\arabic*)}]
      \item reads the string generated so far,
      \item simulates a step of the \eMDMAcr transformer on the string, and
      \item writes the updated values into a \emph{new} portion of the padding space.
   \end{enumerate*}
   
   More precisely, we can implement the \sMDMAcr transformer as follows:
   \begin{enumerate}[label=(\arabic*)]
      \item The \sMDMAcr transformer uses $\timesteps \padlen$ masked symbols to store the generated symbols at each of the $\timesteps$ unmasking steps of the \eMDMAcr transformer.
      \item The planner is implemented as the transformer from \cref{lem:select-block} that acts independently of the input string and uses positional encodings to select the next $\padlen$ positions to unmask.
      \item The predictor is implemented as a transformer that predicts the values of the $\padlen$ masked symbols based on the current input string.
      It reads the values from the padding space by attending to the last $\padlen$ unmasked positions in the padding space analogous to the dump-decode-read mechanism from \cref{lem:idenitity-dump}.
      The only difference is that the predictor has to attend to the \emph{last} unmasked block of decoded values.
      This can be done with a two additional transformer layers by 
      \begin{enumerate*}[label=\textit{(\alph*)}]
         \item ignoring all masked symbols and
         \item ignoring positions with identical positional encodings to their right.
      \end{enumerate*} 
   \end{enumerate}   
\end{proof}

\section{Theoretical gadgets} \label{app:theoretical-gadgets} 
This section contains various theoretical gadgets that are used in the proofs of the main results.
Not all of these are novel; some are modified restatements from their original sources.

In the following, $\strlen \in \N$ always refers to the length of the original input string. 
If the string is additionally padded, the number of padding symbols is denoted by $\padlen$, meaning that the entire input to the transformer is of length $\strlen + \padlen$.

\subsection{Positional Encodings} \label{app:positional-encodings}
Uniquely identifying positions in a string requires the ``volume'' of the representation space to grow with the string length.
In the case of finite-precision logarithmic-width transformers, this is achieved with positional encodings that encode the binary representation of the position in the string.
The following lemma follows from the definition of fixed-point arithmetic, and the rounding and thresholding applied therein.
\begin{restatable}{lemma}{FinitePrecisionPropertiesOurVersion} \label{lem:our-finite-precision-properties}
   Let $x \in \F_\numPrec$ for some $\numPrec \in \N$ such that $x > \log{2} (\numPrec + 1)$.
   Then, it holds that
   \begin{subequations}
      \begin{align}
         \exp(x) &= \maxFNum, \\
         \exp(-x) &= 0.
      \end{align}
   \end{subequations}
\end{restatable}

\cref{lem:finite-precision-properties,lem:pos-enc-properties} readily follow from \cref{lem:our-finite-precision-properties}.
\begin{restatable}[\text{\citet[][Lemmata E.1 and E.2]{li2024chain}}]{lemma}{FinitePrecisionProperties} \label{lem:finite-precision-properties}
   For $\maxFNum$ from \cref{app:fixed-point-arithmetic}, it holds that
   \begin{subequations}
      \begin{align}
         \exp(\maxFNum) &= \maxFNum, \\
         \exp(-\maxFNum) &= 0.
      \end{align}
   \end{subequations}
\end{restatable}
\begin{restatable}[\text{\citet[][Lem. E.3]{li2024chain}}]{lemma}{PosEncProperties} \label{lem:pos-enc-properties}
   For $\strlen \in \N$, $\stridx \in \NTo{\strlen}$, define the vectors $\vq_\stridx \in \R^{2 \ceil{\log\strlen}}$ and $\vk_\stridx \in \R^{2 \ceil{\log\strlen}}$ as follows:
   \begin{subequations}
      \begin{align}
         \vq_\stridx &\defeq \maxFNum \cdot (\interleaveFun{\sbinFun{\stridx}}{\one_{\ceil{\log\strlen}}}) \\
         \vk_{\stridx'} &\defeq \interleaveFun{\sbinFun{\stridx'}}{(-\one_{\ceil{\log\strlen}})}.
      \end{align}
   \end{subequations}
   Then, it holds that
   \begin{equation}
      \inner{\vq_\stridx}{\vk_{\stridx'}} =
      \begin{cases}
         0 & \ifcondition \stridx = \stridx' \\
         -\maxFNum & \otherwisecondition.
      \end{cases}
   \end{equation}
   Thus, 
   \begin{equation}
      \exp(\inner{\vq_\stridx}{\vk_{\stridx'}}) =
      \begin{cases}
         1 & \ifcondition \stridx = \stridx' \\
         0 & \otherwisecondition.
      \end{cases}
   \end{equation}
\end{restatable}
The following slight generalization of \cref{lem:pos-enc-properties} is also easy to show.
\begin{restatable}{lemma}{PosEncPropertiesGen} \label{lem:our-pos-enc-properties}
   For $\strlen \in \N$, $\stridx \in \NTo{\strlen}$, define the vectors $\vq_\stridx \in \R^{2 \ceil{\log\strlen}}$ and $\vk_\stridx \in \R^{2 \ceil{\log\strlen}}$ as follows:
   \begin{subequations}
      \begin{align}
         \vq_\stridx &\defeq \sfrac{\maxFNum}{m} \cdot (\interleaveFun{\sbinFun{\stridx}}{\one_{\ceil{\log\strlen}}}) \\
         \vk_{\stridx'} &\defeq \interleaveFun{\sbinFun{\stridx'}}{(-\one_{\ceil{\log\strlen}})}.
      \end{align}
   \end{subequations}
   Then, it holds that
   \begin{equation}
      \inner{\vq_\stridx}{\vk_{\stridx'}} =
      \begin{cases}
         0 & \ifcondition \stridx = \stridx' \\
         x & \otherwisecondition.
      \end{cases}
   \end{equation}
   where $x \leq -\sfrac{\textcolor{ETHRed}{2}\maxFNum}{m}$.
\end{restatable}

Our constructions heavily rely on specific positional encodings.
We assume that this positional information comes from an outside source and is not computed by the transformer model directly. 
We note that this is in contrast to some existing work with causally masked transformers where the positional encodings are inferred by the model itself \citep{yang2024masked,li2025characterizingexpressivitytransformerlanguage,jerad2025uniquehardattentiontale}.
This matters for multiple reasons.
First, including positional information from an external source allows us to decouple the computation of positional information from the computations performed by the transformer model.
Furthermore, it facilitates providing the model with structured information that it would not be able to compute on its own.
While this could be abused to give the model unrealistic computational power, we assume that the positional information is easily computable and thus realistic.

In that vein, it is interesting to consider what the minimal amount of positional information is that the model has to be provided with for it to be able to construct the useful positional encodings.
This is described by the following lemma.
At a high level, it says that our positional encodings require the binary encodings of the relevant numbers, along with any modular and division operations performed on them.
Addition, thresholding, and multiplication by a power of two, in contrast, can be performed by the MLPs in the transformer model.
\begin{restatable}[Arithmetic operations performed by MLPs]{lemma}{MLPArithmeticOperationsLemma} \label{lem:mlp-arithmetic-operations}
   Let $m, n \in \N$. 
   Then, given the binary encodings of the numbers, $\binFun{m}, \binFun{n}$, there exist MLPs that can compute the following operations:
   \begin{enumerate}[label=\textit{(\alph*)}]
      \item computing the signed binary encoding $\sbinFun{m}$,
      \item computing the sum $\binFun{m} + \binFun{n}$ and the difference $\binFun{m} - \binFun{n}$,
      \item computing $\binFun{2^k m}$ for $k \in \N$,
      \item computing the indicator function $\ind{m \geq 0}$,
      \item computing the indicator function $\ind{m = 0}$, and
      \item computing the positive-part function $\posPart{m}$.
   \end{enumerate}
\end{restatable}
\begin{proof} \
   \begin{enumerate}[label=\textit{(\alph*)}]
      \item The transformation $\sbinFun{m} \defeq 2 \binFun{m} - \one_{\ceil{\log m}}$ is an affine function that can be implemented by the affine part of the MLP followed by the identity function, which can be implemented by an MLP as well.
      \item It is easy to implement \andGate and \notGate gates by an MLP. 
      This suffices to implement any $\ACZero$ (and thus addition and subtraction) circuit.
      \item The transformation $\binFun{m} \mapsto \binFun{2^k m}$ can be performed by shifting the binary representation of $\binFun{m}$ to the left by $k$ positions, which can be implemented by a linear transformation.
      \item The indicator function $\ind{m \geq 0}$ can be computed by checking the sign bit of the signed binary representation $\sbinFun{m}$, which can be implemented by an MLP.
      \item The indicator function $\ind{m = 0}$ can be computed by checking if all bits of the binary representation $\binFun{m}$ are zero, which can also be implemented by an MLP.
      \item Computing $\posPart{m}$ can be done by first computing $\ind{m \geq 0}$ with \textit{(d)} and then conditionally outputting $m$ or $0$ based on the result, which can be implemented by an MLP similar to \cref{lem:projection-mlp}.
   \end{enumerate}
\end{proof}

\paragraph{Simplification of positional encodings.}
\cref{lem:mlp-arithmetic-operations} leads us to conclude that the relatively complicated positional encodings used in the following proofs can be thought of as simple transformations of the ``basic'' information captured by
\begin{itemize}[noitemsep,topsep=0pt]
   \item $\binEnc{\stridx}$,
   \item $\binEnc{\strlen}$,
   \item $\binEnc{\stridx \mod m}$ for some relevant $m \in \N$, and
   \item $\binEnc{\sfrac{\stridx}{m}}$ for some relevant $m \in \N$.
\end{itemize}
This makes the positional encodings in our constructions streamlined and easily computable. 
While somewhat non-standard, we note that positional encodings based on both the symbol position $\stridx$ and string length $\strlen$ are common in theoretical literature \citep{chiang-cholak-2022-overcoming}.

Despite being simple to compute, these positional encodings are powerful enough to allow the transformer to uniquely identify positions in the string and to perform useful computations based on them.
In a sense, the inclusion of this information is also necessary, as the operations such as division and modular arithmetic---including the computation of the binary encodings---lie outside of $\ACZero$ and thus cannot be performed by finite-precision transformers \citep{li2024chain}.
We note, however, that more expressive transformers such as those with logarithmic precision could possibly implement the required functions to compute the information in $\binEnc{\stridx}$, $\binEnc{\stridx \mod m}$, and $\binEnc{\sfrac{\stridx}{m}}$ from $\stridx$ and $\strlen$ directly since, unlike fixed-precision transformers, they are not constrained to $\ACZero$ operations \citep{li2024chain}.\footnote{We also note that the modular information in our encodings resembles the periodic nature of original sinusoidal positional encodings used by the transformer architecture \citep{NIPS2017_3f5ee243}. Moreover, the modular nature of such positional encodings has been analyzed by theoretical work before and is known to increase the expressivity of certain idealizations of transformers \citep{li2025characterizingexpressivitytransformerlanguage,jerad2025uniquehardattentiontale}.}
The simplicity and uniformity of these encodings lies in contrast to more complex (non-uniform) positional encodings that directly serialize the circuits to be simulated when analyzing expressivity lower bounds \citep{li2024chain,saunshi2025reasoninglatentthoughtspower,london2025pausetokensstrictlyincrease}.

\subsection{Useful attention patterns} \label{app:indexing-positions}
The following lemmata describe how a transformer layer can either ignore or exclusively focus on specific positions in the input string.
\begin{restatable}[Ignoring Marked Positions with a Transformer]{lemma}{IgnoreMarkedPositionsTransformer} \label{lem:ignore-positions}
   Let $\strlen, \hiddDim \in \N$, $\sN \subseteq \NTo{\strlen}$, and let $\residStream \in \R^{\strlen \times \hiddDim}$ be a matrix representing the residual stream such that
   \begin{equation} \label{eq:ignore-positions-condition}
      \onehot{\padSym} \in \residStream_{\stridx, :} \iff \stridx \in \sN.
   \end{equation}
   Here, the notation $\onehot{\padSym} \in \residStream_{\stridx, :}$ means that the vector $\residStream_{\stridx, :}$ contains the one-hot encoding of the symbol $\padSym$ at position $\stridx$.
   Further, let $\mG \defeq \residStream_{\NTo{\strlen} \setminus \sN, :} \in \R^{(\strlen - |\sN|) \times \hiddDim}$, where $\residStream_{\NTo{\strlen} \setminus \sN, :}$ denotes the projection of the matrix $\residStream$ onto the rows \emph{not} in $\sN$.
   Finally, let $\tflayer$ be a transformer layer.
   Then, there exists a logarithmic-width transformer layer $\tflayer'$ such that it holds for $\mG' \defeq \tflayer(\mG) \in \R^{(\strlen - |\sN|) \times \hiddDim}$ and $\residStream' \defeq \tflayer'(\residStream) \in \R^{\strlen \times \hiddDim}$ that
   \begin{equation}
      \mG' = \residStream'_{\NTo{\strlen} \setminus \sN, :}.
   \end{equation}
\end{restatable}
Informally, \cref{lem:ignore-positions} states that a transformer layer can ignore positions containing one-hot encodings of specific ``marker'' symbols, such as additional symbols not in the original alphabet.
Here, ignoring means that the content of the positions with the marker symbols does not affect the output of the transformer layer at other positions.
\begin{proof}
   Notice that, since $\residStream$ and $\mG$ match on all positions not in $\sN$, \emph{ignoring} the positions in $\sN$ (marked by $\padSym$) by $\tflayer'$ will ensure that the outputs of the two layers $\tflayer$ and $\tflayer'$ are identical on the positions not in $\sN$.
   We now construct a transformer layer that ignores the contributions of rows marked by $\onehot{\padSym}$.
   To do so, we modify each attention head of $\tflayer$ such that the head computes its attention scores with queries and keys of the form 
   \begin{subequations}
      \begin{align}
         \vq'_\stridx &\defeq \cdot 
         \begin{pmatrix}
            \vq_\stridx \\
            -\maxFNum \cdot \onehot{\padSym} \\
            -\maxFNum \cdot \onehot{\padSym} \\
         \end{pmatrix} \\
         \vk'_{\stridx'} &\defeq \phantom{\maxFNum \cdot} 
         \begin{pmatrix}
            \vk_{\stridx'} \\
            \onehot{\sym_{\stridx'}} \\
            \onehot{\sym_{\stridx'}}
         \end{pmatrix}
      \end{align}
   \end{subequations}
   where $\vq_\stridx$ and $\vk_{\stridx'}$ are the original head's query and key vectors of $\tf$ at position $\stridx$ and $\stridx'$, respectively, and $\onehot{\sym_{\stridx'}}$ is the one-hot encoding of the symbol at position $\stridx'$.
   We can then compute the dot product of the two vectors as
   \begin{equation}
      \inner{{\vq'}_\stridx}{{\vk'}_{\stridx'}} = \inner{\vq_\stridx}{\vk_{\stridx'}} - \maxFNum \cdot \ind{\sym_{\stridx'} = \padSym} - \maxFNum \cdot \ind{\sym_{\stridx'} = \padSym}.
   \end{equation}
   Thus, if the symbol at position $\stridx'$ is not $\padSym$, the attention score is $\inner{\vq_\stridx}{\vk_{\stridx'}}$ and the head behaves as it did in $\tf$.
   If the symbol at position $\stridx'$ is $\padSym$, the last two components of the vectors $\vq'_\stridx$ and $\vk'_{\stridx'}$ ensure that the exponentiated value of the attention score becomes $0$, thus not contributing to the attention weights.
   $\tf'$ can thus simulate $\tf$ on the rest of the positions.
\end{proof}

\begin{restatable}[Focusing on Marked Positions with a Transformer]{lemma}{FocusOnMarkedPositionsTransformer} \label{lem:focus-on-marked-positions}
   Let $\strlen, \hiddDim \in \N$, $R\colon \NTo{\strlen} \to \NTo{\strlen}$, $r \in \NTo{\strlen}$, $\sN \defeq \inv{R}(r) \subseteq \NTo{\strlen}$, and let $\residStream \in \R^{\strlen \times \hiddDim}$ be a matrix representing the residual stream such that
   \begin{equation} \label{eq:focus-on-positions-condition}
      \fullBinEnc{R(\stridx)} \in \residStream_{\stridx, :} \quad \text{for all } \stridx \in \NTo{\strlen}
   \end{equation}
   Here, the notation $\fullBinEnc{R(\stridx)} \in \residStream_{\stridx, :}$ means that the vector $\residStream_{\stridx, :}$ contains the signed binary encoding (cf. \cref{sec:preliminaries}) of $R(\stridx)$.
   Further, let $\mG \defeq \residStream_{\sN, :} \in \R^{|\sN| \times \hiddDim}$, where $\residStream_{\sN, :}$ denotes the projection of the matrix $\residStream$ onto the rows in $\sN$.
   Finally, let $\tflayer$ be a transformer layer.
   Then, there exists a logarithmic-width transformer layer $\tflayer'$ such that it holds for $\mG' \defeq \tflayer(\mG) \in \R^{|\sN| \times \hiddDim}$ and $\residStream' \defeq \tflayer'(\residStream) \in \R^{\strlen \times \hiddDim}$ that
   \begin{equation}
      \mG' = \residStream'_{\sN, :}.
   \end{equation}
\end{restatable}
Informally, \cref{lem:focus-on-marked-positions} states that a transformer layer can focus on positions containing signed binary encodings of a number $r$ computed as a function of the position while ignoring the rest of the positions.
\begin{proof}
   The idea of the construction of $\tflayer'$ is similar to that of \cref{lem:ignore-positions}, but instead of ignoring the positions in $\sN$, we want the transformer layer to focus on them while ignoring the rest of the positions.
   This can be done by including $\fullBinEnc{R(\stridx)}$ in the positional encodings of the attention heads and then using the key and query vectors of the form
   \begin{subequations}
      \begin{align}
         \vq'_\stridx &\defeq 
         \begin{pmatrix}
            \vq_\stridx \\
            \sfrac{\maxFNum}{2} \cdot (\interleaveFun{\sbinFun{r}}{\one_{\ceil{\log\strlen}}}) \\
            \sfrac{\maxFNum}{2} \cdot (\interleaveFun{\sbinFun{r}}{\one_{\ceil{\log\strlen}}}) \\
         \end{pmatrix} \\
         \vk'_{\stridx'} &\defeq \phantom{\sfrac{\maxFNum}{2} \cdot} 
         \begin{pmatrix}
            \vk_{\stridx'} \\
            \interleaveFun{\sbinFun{R(\stridx')}}{(-\one_{\ceil{\log\strlen}})} \\
            \interleaveFun{\sbinFun{R(\stridx')}}{(-\one_{\ceil{\log\strlen}})} \\
         \end{pmatrix}
      \end{align}
   \end{subequations}
   where $\vq_\stridx$ and $\vk_{\stridx'}$ are the original head's query and key vectors of $\tf$ at position $\stridx$ and $\stridx'$, respectively, and $\maxFNum$ is the maximal representable number in the fixed-point arithmetic (which might depend on the string length $\strlen$).
   We can then compute the dot product of the two vectors as
   \begin{subequations}
      \begin{align}
         \inner{{\vq'}_\stridx}{{\vk'}_{\stridx'}} 
         &= \inner{\vq_\stridx}{\vk_{\stridx'}} + \underbrace{\sfrac{\maxFNum}{2} \cdot \inner{(\interleaveFun{\sbinFun{r}}{\one_{\ceil{\log\strlen}}})}{(\interleaveFun{\sbinFun{R(\stridx')}}{(-\one_{\ceil{\log\strlen}})})}}_{\defeq G} \label{eq:focusing-attention-dot-product} \\
         & \phantom{= \inner{\vq_\stridx}{\vk_{\stridx'}}} \ + \underbrace{\sfrac{\maxFNum}{2} \cdot \inner{(\interleaveFun{\sbinFun{r}}{\one_{\ceil{\log\strlen}}})}{(\interleaveFun{\sbinFun{R(\stridx')}}{(-\one_{\ceil{\log\strlen}})})}}_{\defeq G}  \nonumber
      \end{align}
   \end{subequations}
   
   Note that \cref{eq:focusing-attention-dot-product} uses fixed-point arithmetic.
   We compute the inner product in \cref{eq:focusing-attention-dot-product} by analyzing individual cases:
   \begin{enumerate}[label=\arabic*.]
      \item \textbf{Case 1:} $R(\stridx') \neq r$.
      
      All intermediate computations of \cref{eq:focusing-attention-dot-product} are thresholded at $\min(\maxFNum, \inner{\vq_\stridx}{\vk_{\stridx'}} + \sfrac{\maxFNum}{2})$.
      In particular, by \cref{lem:our-pos-enc-properties}, the value after adding the first term $G$ is at most $\min(\maxFNum, \inner{\vq_\stridx}{\vk_{\stridx'}} + \sfrac{\maxFNum}{2}) - 2\sfrac{\maxFNum}{2} \leq \maxFNum -   \maxFNum = 0$.
      After adding the second term $G$, the value is at most $-\maxFNum$, resulting in a $0$ exponentiated attention score, as required.
      \item \textbf{Case 2:} $R(\stridx') = r$.
      We analyze three sub-cases based on the value of $\inner{\vq_\stridx}{\vk_{\stridx'}}$.
      
      \begin{enumerate}[label=\arabic*.]
         \item \textbf{Sub-case 2a:} $\abs{\inner{\vq_\stridx}{\vk_{\stridx'}}} < \log{2} (\numPrec + 1)$.
         All intermediate computations in \cref{eq:focusing-attention-dot-product} are bounded by $\log{2} (\numPrec + 1) + \sfrac{\maxFNum}{2}$ in absolute value, so they fall within the range of $\F_\numPrec$. 
         Moreover, addition of $\sfrac{\maxFNum}{2}$ can be exactly represented in $\F_\numPrec$.
         This makes addition in \cref{eq:focusing-attention-dot-product} associative and commutative.
         By \cref{lem:pos-enc-properties}, the terms $G$ in \cref{eq:focusing-attention-dot-product} are $0$, meaning that the final attention score equals $\inner{\vq_\stridx}{\vk_{\stridx'}}$.
         \item \textbf{Sub-case 2b:} $\inner{\vq_\stridx}{\vk_{\stridx'}} \geq \log{2} (\numPrec + 1)$.
         In this case, the intermediate computations of \cref{eq:focusing-attention-dot-product} are either exact or thresholded at $\maxFNum$.
         In both cases, the exponent of the resulting attention score is $\maxFNum$ by \cref{lem:our-finite-precision-properties}, preserving the attention score.
         \item \textbf{Sub-case 2c:} $\inner{\vq_\stridx}{\vk_{\stridx'}} \leq -\log{2} (\numPrec + 1)$.
         In this case, all intermediate computations are representable in $\F_\numPrec$ analogously to the case 2a.
         The attention score is therefore preserved.
      \end{enumerate}
   \end{enumerate}
   
   Taken together, this means that the attention scores between positions $\stridx$ and $\stridx'$ of $\tflayer'$ are identical to those of $\tflayer$ on the positions in $\sN$, while the attention scores on the rest of the positions are $0$.
   This completes the proof.
\end{proof}

\begin{restatable}[Detecting a symbol occurrence]{lemma}{DetectionLemma} \label{lem:detection}
   Let $\str \in \strings$ and $\sym \in \alphabet$. 
   Then, there exists a single-layer unmasked fixed-precision logarithmic-width transformer $\tf$ such that, on input $\str$, an entry of its final residual stream contains the entry $\ind{\sym \in \str}$.
\end{restatable}
\begin{proof}[Proof sketch.]
   Note that $\tf$ cannot use the commonly-used \emph{exact} uniform attention over all symbols to detect $\ind{\sym \in \str}$ due to fixed precision.
   Nevertheless, rounded uniform attention suffices. 
   By attending to all symbols in the string with weight 1, the denominator of the attention scores is at most $\maxFNum$.
   Using one-hot encodings of symbols $\sym_\stridx$ as the attention values $\vv_\stridx$, it is easy to see that the final contextual representation at the final position will have a positive value at the entry corresponding to $\sym$ if and only if $\sym \in \str$, since $\sfrac{c}{\maxFNum} > 0$ for any $c \geq 1$.
   This condition can be checked by the MLP applied after the attention aggregation operation.   
\end{proof}

\begin{restatable}{lemma}{SelectBlockLem} \label{lem:select-block}
   Let $\str \in \strings$ be a string of length $\strlen$ and $\padlen \timesteps$ the number of padding symbols where $\padlen, \timesteps = \polyFun{\strlen}$.
   There exists a fixed-precision and logarithmic-width transformer $\planner \colon \kleene{\maskAlphabet} \to \kleene{\{\rejectSym, \acceptSym\}}$ that, given $\str$ and the current partially masked string $\thinkstr^{(t)}$, selects the next $\padlen$ positions to unmask by outputting $\acceptSym$ for the next $\padlen$ positions to unmask and $\rejectSym$ for all other positions.
\end{restatable}
\begin{proof}
   The idea of the construction is for $\planner$ to 
   \begin{enumerate*}[label=\textit{(\arabic*)}]
      \item output $\rejectSym$ for any position that does not contain the padding symbol $\padSym$, and
      \item output $\acceptSym$ for the first $\padlen$ positions that contain the padding symbol $\padSym$.
   \end{enumerate*}
   Step \textit{(1)} can be implemented by checking whether the symbol at the current position is $\padSym$ and outputting $\rejectSym$ otherwise.
   Step \textit{(2)} can be implemented by attending to position $\padlen$ positions back and outputting $\acceptSym$ if the symbol at that position $\neq \padSym$ and $\rejectSym$ otherwise.
   These steps can be performed by two attention heads in a single transformer layer using the positional encodings
   \begin{equation} \label{eq:pos-enc-select-block}
      \posEncFun{\stridx, \strlen} \defeq 
      \begin{pmatrix}
         \sbinFun{\stridx} \\
         \sbinFun{\stridx - \padlen} \\
      \end{pmatrix} \in \set{0, 1}^{\bigOFun{\log\strlen}}.
   \end{equation}
\end{proof}

\begin{restatable}[Converting a binary representation into a positional encoding]{lemma}{BinaryEncodingLemma} \label{lem:binary-encoding}
   Let $\strlen \in \N$ and $\stridx \in \NTo{\strlen}$.
   Then, there exists an unmasked fixed-precision logarithmic-width padded looped transformer $\tf$ such that, on input $\texttt{\& } \binEnc{\stridx}$, after $\ceil{\log \strlen}$ iterations, the residual stream at position $\ceil{\log \strlen} + 1$ contains the value $\binEnc{\stridx}$.
\end{restatable}
\begin{proof}[Proof sketch]
   The transformer $\tf$ has to convert the binary representation $\binEnc{\stridx}$ of $\stridx$ spread across $\ceil{\log \strlen}$ positions in the input string into a single $\ceil{\log \strlen}$-dimensional binary vector in the residual stream.
   This is done as follows:
   \begin{enumerate}[label=\arabic*.,noitemsep]
      \item In the first layer, each symbol $\sym_{\stridx'} \in \set{0, 1}$ checks if it is immediately preceded by the $\texttt{\&}$ symbol, which denotes the beginning of the pointer in the string. 
      If it is, $\sym_{\stridx'}$ stores $\ve_1$ and $\vd_1 \defeq \sym_{\stridx'} \ve_1$ in designated parts of its residual stream. 
      Here, $\ve_1$ is the first unit vector of $\R^{\ceil{\log \strlen}}$.
      \item In the subsequent layers $\layeridx \in \set{2, \ldots, \ceil{\log\strlen}}$, each symbol $\sym_{\stridx'}$ checks if the entry $\ve_{\layeridx - 1}$ has already been written to the designated space of the previous symbol's residual stream.
      If it has, $\sym_{\stridx'}$ copies and shifts $\ve_{\layeridx - 1}$ into $\ve_{\layeridx}$, and stores $\ve_{\layeridx}$ and $\vd_{\layeridx} \defeq \vd_{\layeridx - 1} + \sym_{\stridx'} \ve_{\layeridx}$ in designated parts of its residual stream.
   \end{enumerate}
   After $\ceil{\log \strlen}$ layers, the residual stream at position $\ceil{\log \strlen} + 1$ thus contains $\binEnc{\stridx}$.
\end{proof}

\subsection{Neural network constructions}
\begin{restatable}[\text{\citet[][Lem. B.6]{saunshi2025reasoninglatentthoughtspower}}]{lemma}{MLPForArgmaxLem} \label{lem:mlp-for-argmax}
   For every $\hiddDim \in \N$ and precision $\numPrec \in \N$, there exists a $\ReLU$-activated MLP $\mlp\colon \R^\hiddDim \to \set{0, 1}^\hiddDim$ such that for any $\vx \in \F^{\hiddDim}_\numPrec$, if there is $\dimidx \in \NTo{\hiddDim}$, such that $\evx_d > \max_{j \in \NTo{\hiddDim} \setminus \set{\dimidx}} \evx_{j}$, then $\mlp(\vx) = \ve_d$, the $\dimidx\textsuperscript{th}$ unit vector.
\end{restatable}
\begin{proof}
   The proof is based on the construction of a 3-layer $\ReLU$ network that computes the $\argmax$ of a vector $\vx \in \F^{\hiddDim}_\numPrec$.
   The first layer computes the differences between each pair of elements in $\vx$.
   The second layer computes the maximum of these differences.
   The third layer then checks if the maximum difference is greater than zero, indicating that there is a unique maximum element in $\vx$.
   
   More concretely, define 
   \begin{equation}
      g_\dimidx \defeq 2^\numPrec \cdot \ReLUFun{2^{-\numPrec} - \sum_{j \neq \dimidx} \ReLU(\evx_j - \evx_\dimidx)}, 
   \end{equation}
   which can be computed by a 3-layer $\ReLU$ network.
   We have that $g_\dimidx = 1$ if and only if $\evx_\dimidx > \max_{j \neq \dimidx} \evx_j$, or, equivalently, if and only if $\evx_\dimidx - \max_{j \neq \dimidx} \evx_j \geq 2^{-\numPrec}$.
   Indeed if $\evx_\dimidx - \max_{j \neq \dimidx} \evx_j \geq 2^{-\numPrec}$, we have that 
   \begin{equation}
      g_\dimidx = 2^\numPrec \cdot \ReLUFun{2^{-\numPrec} - \sum_{j \neq \dimidx} \ReLU(\evx_j - \evx_\dimidx)} = 1.
   \end{equation}
   In contrast, for $\dimidx' \neq \dimidx$, we have that $\evx_{\dimidx'} - \evx_\dimidx < 2^{-\numPrec}$, and thus
   \begin{equation}
      g_{\dimidx'} = 2^\numPrec \cdot \ReLUFun{2^{-\numPrec} - \sum_{j \neq \dimidx'} \ReLU(\evx_j - \evx_{\dimidx'})} \leq 2^\numPrec \cdot \ReLUFun{2^{-\numPrec} - \ReLU(\evx_{\dimidx} - \evx_{\dimidx'})} = 0.
   \end{equation}
\end{proof}

\begin{restatable}{lemma}{ProjectionMLP} \label{lem:projection-mlp} 
   Let $\vx \in \set{0, 1}^\hiddDim$ and $\ve_\dimidx \in \set{0, 1}^\hiddDim$ be the $\dimidx\textsuperscript{th}$ unit vector.
   Then, there exists a $\ReLU$-activated MLP $\mlp\colon \set{0, 1}^{\hiddDim + \hiddDim} \to \set{0, 1}^\hiddDim$ such that
   \begin{equation}
      \mlp(\vx, \ve_\dimidx) = \inner{\vx}{\ve_\dimidx} = \evx_\dimidx.
   \end{equation}
\end{restatable}
\begin{proof}
   We have that 
   \begin{equation}
      \evx_\dimidx = \inner{\one}{\mleft(\ReLUFun{\vx - \mleft(\one - \ve_\dimidx\mright)}\mright)}
   \end{equation}
   where $\one$ is the all-ones vector of length $\hiddDim$.
   This can be implemented by a $\ReLU$-activated MLP.
\end{proof}

\subsection{Masked and unmasked transformers} \label{app:masked-unmasked-transformers}

\begin{restatable}[Unmasked to causally masked; \text{\citet[][Lem. 1]{merrill2025exactexpressivepowertransformers}}]{lemma}{MaskConversionLemma}\label{lem:mask-conversion}
   Let $\tf$ be an unmasked fixed-precision logarithmic-width transformer with $\numlayers$ layers.
   Then, there exists a fixed-precision logarithmic-width causally-masked transformer $\tf'$ with $\numlayers$ layers such that for any input string $\str \in \strings$ of length $\strlen$ padded with $\numlayers (\strlen - 1)$ padding symbols, the representations $\hiddMtx^{\prime(\numlayers)}_{(\numlayers - 1) \strlen : \numlayers \strlen}$ computed by $\tf'$ on $\str$ equal the final representations $\hiddMtx^{(\numlayers)}$ computed by $\tf$.
\end{restatable}
\begin{proof}
   We adapt the proof of \citet[][Lem. 1]{merrill2025exactexpressivepowertransformers} to our setting.
   The idea is for $\tf'$ to unroll the computation of $\hiddMtx^{(\numlayers)}$ into a sequence of $\numlayers$ ``blocks'' of padding, each of width $\strlen$. 
   Each block will attend to the previous block---representing the values in the preceding layer---and will thus be able to see all symbols despite causal masking.
   To do so, $\tf'$ uses additional positional encodings of the form   
   \begin{equation}
      \posEncFun{\stridx, \strlen} \defeq 
      \begin{pmatrix}
         \sbinFun{\layeridx_\stridx} \\
         \sbinFun{\layeridx_\stridx - 1} \\
      \end{pmatrix} \in \set{0, 1}^{2 \ceil{\log\strlen}}
   \end{equation}
   Here, $\layeridx_\stridx \defeq \floor{\sfrac{\stridx}{\numlayers}} + 1$ represents the layer that each padding position belongs to.
   To construct $\tf'$, we then modify each original head from $\tf$ with \cref{lem:focus-on-marked-positions} to ensure that the attention is focused on the correct padding positions, where the attention head computes the same function as the one in $\tf$. 
\end{proof}

\begin{restatable}[Causally masked to unmasked]{lemma}{UnmaskConversionLemma}\label{lem:unmask-conversion-fp}
   Let $\tf$ be an $\numlayers$-layer finite-precision logarithmic-width causally-masked transformer. 
   Then, there exists a finite-precision logarithmic-width unmasked transformer $\tf'$ with $2 \numlayers + 1$ layers such that for any input string $\str \in \strings$ of length $\strlen$ padded with $(\strlen - 1)\strlen$ padding symbols, the representations $\hiddMtx^{\prime(\numlayers)}_{(\strlen - 1) \strlen : \strlen^2}$ computed by $\tf'$ on $\str$ equal the final representations $\hiddMtx^{(\numlayers)}$ computed by $\tf$.\footnote{This simulation is somewhat inefficient in that only a subset of the $\strlen$ positions are used at each of the $\strlen$ blocks (specifically, $\stridx$ positions in the $\stridx\textsuperscript{th}$ block). While this could be made more efficient with a more sophisticated construction, the asymptotic complexity would remain quadratic in $\strlen$.}
\end{restatable}
\begin{proof}
   The idea is for $\tf'$ to unroll the computation of $\hiddMtx^{(\numlayers)}$ into a sequence of $\strlen$ ``blocks'' of padding, each of width $\strlen$. 
   Each block will compute the representation of one of the symbols in the string.
   To do so, $\tf'$ uses additional\footnote{By additional, we mean that these positional encodings are appended to the ones used by $\tf$.} positional encodings of the form
   \begin{equation} \label{eq:additional-pos-enc-1}
      \posEncFun{\stridx, \strlen} \defeq 
      \begin{pmatrix}
         \sbinFun{\paddingBlock_\stridx} \\
         \sbinFun{\relativePos_{\stridx}}  \\
         \ind{\stridx \le \strlen}  \\
         \ind{\paddingBlock_\stridx \ge \relativePos_{\stridx}}  \\
         \ind{\paddingBlock_\stridx = \relativePos_{\stridx}}  \\
      \end{pmatrix} \in \set{0, 1}^{\bigOFun{\log\strlen}}.
   \end{equation}
   Here, $\paddingBlock_\stridx \defeq \floor{\sfrac{\stridx}{\strlen}} + 1$ represents the block that position $\stridx$ falls into and $\relativePos_{\stridx} \defeq (\stridx \mod \strlen) + 1$ represents the position within that layer.
   
   $\tf'$ then processes a string $\str \in \strings$ of length $\strlen$ padded with $\strlen^2$ padding symbols as follows.
   \begin{enumerate}[label=(\arabic*)]
      \item $\tf'$ uses an additional ``copy'' layer to copy the input symbols from the first $\strlen$ positions to the residual stream for access in later layers.
      In particular, each position $\stridx \in \NTo{\strlen^2}$ is assigned the value of the input symbol at position $\relativePos_{\stridx}$.
      This can be done by the symbol at position $\stridx$ attending to the symbol at position $\relativePos_{\stridx}$ in the input string, i.e., $\hiddMtx^{(1)}_{\stridx} = \hiddMtx^{(0)}_{\relativePos_{\stridx}}$, \emph{if} $\paddingBlock_\stridx \ge \relativePos_{\stridx}$, which can be ensured by attending only to positions with non-zero entry $\ind{\paddingBlock_\stridx \ge \relativePos_{\stridx}} \sbinFun{\relativePos_{\stridx}}$ in the positional encoding.
      The latter condition ensures that the $\paddingBlock\textsuperscript{th}$ block contains only symbols $\str_{\leq \paddingBlock}$---$\sym_\paddingBlock$ attending to the \emph{entire} block is then equivalent to $\sym_\paddingBlock$ attending to the string with causal attention.
      Concretely, the attention scores $\evs^{(1)}_{\stridx, \stridx'} = \inner{{\vq^{(1)}_{\stridx}}}{\vk^{(1)}_{\stridx'}}$ are computed with query and key vectors
      \begin{subequations}
         \begin{align}
            \vq^{(1)}_\stridx &\defeq \maxFNum \cdot 
            \begin{pmatrix}
               \interleaveFun{\sbinFun{\relativePos_\stridx}}{\one_{\ceil{\log\strlen}}} \\
               -1
            \end{pmatrix} \\
            \vk^{(1)}_{\stridx'} &\defeq \phantom{\maxFNum \cdot} 
            \begin{pmatrix}
               \interleaveFun{\sbinFun{\relativePos_{\stridx'}}}{(-\one_{\ceil{\log\strlen}})} \\
               \ind{\stridx \le \strlen}
            \end{pmatrix}
         \end{align}
      \end{subequations}
      \item Once the symbols have been copied to the appropriate positions, $\tf'$ can simulate one layer of $\tf$ by augmenting its heads with \cref{lem:focus-on-marked-positions} to ensure that the computations at position $\stridx$ are restricted to the block $\paddingBlock_\stridx$, which, as described above, contains information about $\str_{\le \paddingBlock_\stridx}$.\footnote{The augmented attention mechanism additionally downweights positions with $\ind{\paddingBlock_{\stridx'} \ge \relativePos_{\stridx'}} = 1$, which can be done by subtracting $\maxFNum$ from the attention score.}
      This ensures that the attention scores are non-zero only for 
      \begin{enumerate*}[label=\textit{(\roman*)}]
         \item positions $\stridx$ in the same block $\paddingBlock_\stridx$ as $\stridx'$, and
         \item positions that should be unmasked in the current block.
      \end{enumerate*}
      \item After simulating the layer from $\tf$ in step (2), the contextual representation in the $\paddingBlock_{\stridx}\textsuperscript{th}$ block contain the information about $\str_{\leq \paddingBlock_{\stridx}}$ computed based on the symbols $\str_{\leq \paddingBlock_{\stridx}}$.
      In particular, the representations of the symbols $\str_{< \paddingBlock_{\stridx}}$ in the $\paddingBlock_{\stridx}\textsuperscript{th}$ block contain information \emph{not} obtained by causal masking since they attend to \emph{all} the symbols $\str_{\leq \paddingBlock_{\stridx}}$ in the $\paddingBlock_{\stridx}\textsuperscript{th}$ block.
      To amend that, an additional transformer layer discards the information about symbols $\str_{< \paddingBlock_{\stridx}}$ in the $\paddingBlock_{\stridx}\textsuperscript{th}$ block by overwriting the representation of $\sym_{\stridx'}$ in the $\paddingBlock_{\stridx}\textsuperscript{th}$ block with the representation of $\sym_{\stridx'}$ in the $\stridx'\textsuperscript{th}$ block for $\stridx' < \paddingBlock_{\stridx}$.
      This is done by attending to the positions $\stridx'$ in which the block index $\paddingBlock_{\stridx'}$ matches the position $\relativePos_{\stridx'}$, i.e., with the query and key vectors
      \begin{subequations}
         \begin{align}
            \vq_\stridx &\defeq \maxFNum \cdot 
            \begin{pmatrix}
               \interleaveFun{\sbinFun{\paddingBlock_\stridx}}{\one_{\ceil{\log\strlen}}} \\
               -1
            \end{pmatrix} \\
            \vk_{\stridx'} &\defeq \phantom{\maxFNum \cdot} 
            \begin{pmatrix}
               \interleaveFun{\sbinFun{\relativePos_{\stridx'}}}{(-\one_{\ceil{\log\strlen}})} \\
               \ind{\paddingBlock_\stridx = \relativePos_{\stridx}}
            \end{pmatrix}
         \end{align}
      \end{subequations}
      which ensures that $\stridx$ uniquely attends to the symbols $\str_{\leq \paddingBlock_\stridx}$ in the $\paddingBlock_{\stridx}\textsuperscript{th}$ block.
   \end{enumerate}
\end{proof}

\section{Proofs} \label{app:proofs}
This section contains the proofs of all novel theoretical results. 
Many constructions in the proofs rely on the theoretical gadgets introduced in \cref{app:theoretical-gadgets}.

\perfectApproximationTheorem*
\begin{proof}
   This is a consequence of the established result that fixed-depth transformers can only compute $\ACZero$ functions \citep[][\textit{inter alia}]{li2024chain,saunshi2025reasoninglatentthoughtspower,london2025pausetokensstrictlyincrease}.
   Note that this is similar to the proof in \citet{liu2025perfectdiffusionmathsftc0,liu2025serialscalinghypothesis} but is simpler due to the focus on discrete predictions directly rather than the continuous modeling of the diffusion process in the latent space.
\end{proof}

\subsection{Proofs of results in \texorpdfstring{\cref{sec:mdms-and-lts}}{Section 3.2}}

\begin{restatable}[\LTsAcr can simulate \eMDMsAcr ]{theorem}{LTsCanSimulateMDMsThm} \label{thm:lts-can-simulate-mdms}
   \begin{equation}
      \MDMClassEdit[\timesteps, \padlen] \subseteq \LTClass[\timesteps, \padlen].
   \end{equation}
\end{restatable}
\begin{proof}
   We can simulate an \eMDMAcr transformer with a \LTAcr transformer by ``composing'' the planner and predictor into a single transformer model.
   This model
   \begin{enumerate}[label=(\arabic*)]
      \item computes the planner's contextual representations while passing the input symbol in the residual stream, 
      \item computes the planner's decision at each position by simulating the $\argmax$ of the planner's output logits as in \cref{lem:mlp-for-argmax},
      \item computes the predictor's contextual representations based on the planner's decision, and
      \item predicts the next symbol at each position by simulating the $\argmax$ of the predictor's output logits as in \cref{lem:mlp-for-argmax}.
   \end{enumerate} 
\end{proof}

\begin{restatable}{lemma}{IdentityDumpLayer} \label{lem:idenitity-dump}
   Let $\tf$ be a fixed-precision and polynomial-width transformer and $\residStream \in \set{0, 1}^{\strlen \times \hiddDim}$ the residual stream of $\tf$ after $\layeridx$ layers on some (possibly padded) input string $\str \in \strings$.
   Then, there exist fixed-precision and polynomial-width transformer layers $\dumpLayer$ and $\readLayer$ with such that $\dumpLayerFun{\residStream} \in \set{0, 1}^{\bigOFun{\log\strlen} \times \hiddDim}$ and 
   \begin{equation}
      \readLayerFun{\decodeStepFun{\dumpLayerFun{\residStream}}}_{: \strlen, :} = \residStream
   \end{equation}
   for some output matrix $\outMtx \in \R^{2 \times \hiddDim}$.
\end{restatable}
\begin{proof}
   The idea of the construction of the layers $\dumpLayer$ and $\readLayer$ is to store the contents of the residual stream in the padding space and then read it out at the next iteration.
   To do that, we allocate $\strlen \cdot \hiddDim$ symbols of additional (masked) padding space in the decoded string.
   Each of the $\strlen$ length-$\hiddDim$ blocks corresponds to a position in the input string and each symbol in the block to a dimension in the hidden representation.
   The layers thus have to be augmented with positional encodings that will allow for the identification of the position in the residual stream and the dimension in the hidden representation.
   This will suffice for $\dumpLayer$ to write out the contents of the residual stream into the padding space and for the reading layer $\readLayer$ to read it out again.
   
   More precisely, $\dumpLayer$ and $\readLayer$ use the following positional encodings:
   \begin{equation}
      \posEncFun{\stridx, \strlen} \defeq 
      \begin{pmatrix}
         \sbinFun{\stridx} \\
         \sbinFun{\paddingBlock_\stridx} \\
         \onehot{\dimidx_\stridx}  \\
         \ind{\stridx \leq \strlen}
      \end{pmatrix} \in \set{0, 1}^{\bigO{\log\strlen}}
   \end{equation}
   to the input of $\dumpLayer$.
   In particular, for a masked padding symbol at position $\stridx$, $\paddingBlock_\stridx \defeq \sfrac{\posPart{\stridx - \strlen}}{\hiddDim}$ and $\dimidx_\stridx \defeq \posPart{\stridx - \strlen} \mod \hiddDim$ correspond to the position in the residual stream and the dimension in the hidden representation that the position will store, respectively.
   $\dumpLayer$ can then be implemented as follows:
   \begin{enumerate}[label=(\arabic*)]
      \item Using $\sbinFun{\paddingBlock_\stridx}$ as the query at masked position $\stridx$ and $\sbinFun{\stridx'}$ as the key at position $\stridx' \leq \strlen$, $\dumpLayer$ can individually identify the corresponding position $\sbinFun{\paddingBlock_\stridx}$ in the residual stream.
      \item Feeding $\hiddState_{\sbinFun{\paddingBlock_\stridx}} \in \set{0, 1}^\hiddDim$ together with $\onehot{\dimidx_\stridx}$ as the value at the masked position $\stridx$ into the MLP, $\dumpLayer$ can write the value of the dimension $\dimidx_\stridx$ of the residual stream at position $\sbinFun{\paddingBlock_\stridx}$ into the padding space by \cref{lem:projection-mlp}.
   \end{enumerate}
   It is then easy to construct the output matrix $\outMtx$ as part of the decoding step $\decodeStep$ such that $\decodeStep(\dumpLayerFun{\residStream})$ decodes the contents of the residual stream.
   $\readLayer$ can then be implemented as follows:
   \begin{enumerate}[label=(\arabic*)]
      \item Include a transformer layer that ignores the attention mechanism and reads the input string and the decoded residual stream values, passes them through the residual connection, and encodes the values into the embedding space with the position-wise MLP.
      In particular, combining the information in $\onehot{\dimidx_\stridx}$ with the information in the input symbols, the MLP can convert the one-hot encoding of the dimension $\dimidx_\stridx$ into the vector $\ehiddState_{\dimidx_\stridx} \ve_{\dimidx_\stridx}$, where $\ehiddState_{\dimidx_\stridx}$ corresponds to the value of the appropriate dimension of the residual stream at the appropriate position.
      \item Using $\sbinFun{\stridx}$ as the query at the input string position $\stridx \leq \strlen$ and $\sbinFun{\sfrac{\posPart{\stridx' - \strlen}}{\hiddDim}}$ as the key, $\readLayer$ can identify all the padding positions that contain the values of the individual dimensions of the residual stream at position $\stridx$.
      The positional encodings ensure that the attention scores satisfy (cf. \cref{lem:finite-precision-properties})
      \begin{equation}
         \inner{\vq_{\stridx}}{\vk_{\stridx'}} = \begin{cases}
            0 & \ifcondition \stridx = \sfrac{\posPart{\stridx' - \strlen}}{\hiddDim} \\
            -\maxFNum & \otherwisecondition.
         \end{cases}
      \end{equation}
      Exponentiating and normalizing the attention scores, the attention mechanism will then only attend to the positions $\stridx'$ that correspond to the position $\stridx$ in the residual stream.
      More concretely, the attention mechanism computes
      \begin{equation}
         \evs_{\stridx, \stridx'} = \frac{\exp(\inner{\vq_{\stridx}}{\vk_{\stridx'}})}{\sum_{\stridx''} \exp(\inner{\vq_{\stridx}}{\vk_{\stridx''}})} = \frac{1}{\sum_{\stridx''} \exp(\inner{\vq_{\stridx}}{\vk_{\stridx''}})} = \frac{1}{\min\mleft(\hiddDim, \maxFNum\mright)} \geq \frac{1}{\maxFNum}
      \end{equation}
      for all $\stridx'$ that correspond to the position $\stridx$ in the residual stream and $\evs_{\stridx, \stridx'} = 0$ otherwise.
      Summing over the values at these positions (which project the constructed vectors $\ehiddState_{\dimidx_\stridx} \ve_{\dimidx_\stridx}$), $\readLayer$ can then reconstruct the value of the residual stream at position $\stridx$ (normalized by $\min\mleft(\hiddDim, \maxFNum\mright)$).
      \item Use the position-wise MLP to convert the normalized value of the residual stream at position $\stridx$ into the vector $\ehiddState_{\dimidx_\stridx} \ve_{\dimidx_\stridx}$.
      This can be done by a $\ReLU$-activated MLP that maps $(-\infty, 0]$ to $0$ and $[\frac{1}{\maxFNum}, \infty)$ to $1$ position-wise.
   \end{enumerate}
\end{proof}

\begin{restatable}[\eMDMsAcr  can simulate \LTsAcr]{theorem}{eMDMsCanSimulateFPLTsThm} \label{thm:emdms-can-simulate-fplts}
   \begin{equation}
      \LTClass[\timesteps, \padlen] \subseteq \MDMClassEdit[\timesteps, (\strlen + \padlen) \hiddDim].
   \end{equation}
\end{restatable}
\begin{proof}
   The proof uses \cref{lem:idenitity-dump} to simulate a single iteration of the \LTAcr transformer loop with a single denoising step in the \MDMAcr transformer.
   In particular, by adding $(\strlen + \padlen) \hiddDim$ padding space, the \MDMAcr has enough room to store the residual stream values, allowing it to decode the values at the next iteration.
   At a high level, the \MDMAcr's planner deterministically selects the entire padding space to unmask or resample, and the predictor
   \begin{enumerate*}[label=\textit{(\arabic*)}]
      \item reads the input string or the currently stored residual stream values,
      \item simulates a single pass of the \LTAcr transformer on the input string $\str$ and the current residual stream values and thus computes the updated value of the residual stream, and
      \item uses \cref{lem:idenitity-dump} to write the updated values into the padding space.
   \end{enumerate*}
   
   More precisely, we can implement the \MDMAcr transformer as follows:
   \begin{enumerate}[label=(\arabic*)]
      \item The \MDMAcr transformer uses $(\strlen + \padlen) \hiddDim$ masked symbols to store the residual stream values.
      \item The planner deterministically outputs a $\acceptSym$ for positions $\stridx > \strlen$ and $\rejectSym$ for positions $\stridx \leq \strlen$.
      \item The predictor predicts the values of the $(\strlen + \padlen) \hiddDim$ symbols based on the current input string and the residual stream values.
      It reads the values from the residual stream by making the first $\strlen$ positions attend to the $(\strlen + \padlen) \hiddDim$ padding positions analogous to the dump-decode-read mechanism from \cref{lem:idenitity-dump}.
   \end{enumerate}
   
   By treating input symbols $\in \alphabet$ separately to the decoded values of the residual stream (which enables the \MDMAcr transformer to simulate both the initial as well as the looping blocks of the \LTAcr transformer), the \MDMAcr transformer can thus simulate the \LTAcr transformer with $\timesteps \padlen \hiddDim$ padding symbols.
\end{proof}

\RegularLanguagesInMDMEfficientThm*
\begin{proof}
   The proof adapts the construction from the proof of \citet[][Thm. 5.1]{saunshi2025reasoninglatentthoughtspower} to the \MDMAcr setting.
   The key difference lies in adapting the positional encodings to allow for the padded tokens to attend to appropriate positions in the residual stream.
   In particular, the first $\ceil{\sfrac{\strlen}{2}}$ positions of the padding space will attend to the $\strlen$ input symbols while the remaining positions in the padding space will attend to the other padding positions of the residual stream in later unmasking steps.
   Concretely, defining $\tilde{\stridx} \defeq \posPart{\stridx - \strlen}$, the positional encodings take the form
   \begin{equation} 
      \posEncFun{\stridx, \strlen} \defeq 
      \begin{pmatrix}
         \sbinFun{\stridx} \\
         \sbinFun{\tilde{\stridx}}  \\
         \sbinFun{\posPart{2 \stridx - \strlen}}  \\
         \sbinFun{\posPart{2 \stridx - \strlen - 1}}  \\
         \sbinFun{\posPart{2 \tilde{\stridx} - \strlen}}  \\
         \sbinFun{\posPart{2 \tilde{\stridx} - \strlen - 1}}  \\
         \ind{\stridx > \strlen}  \\
         \ind{\tilde{\stridx} \geq \sfrac{\strlen}{2}}  \\
      \end{pmatrix} \in \set{0, 1}^{\bigOFun{\log\strlen}}.
   \end{equation}
   This information gives padding tokens enough information to either attend to the input symbols ($\ind{\tilde{\stridx} \geq \sfrac{\strlen}{2}} = 1$) or to the residual stream values, enabling the simulation of the algorithm from \citet[][Thm. 5.1]{saunshi2025reasoninglatentthoughtspower}.
\end{proof}

\subsection{Proofs of results in \texorpdfstring{\cref{sec:mdms-and-cot}}{Section 3.3}} \label{app:proofs-cot}

\MDMsAtLeastAsPowerfulAsCoTFP*
\begin{proof}
   Let $\padlen = \timesteps \padlen'$.
   The idea of the simulation is straightforward: The planner first determines the next $\padlen'$ symbols to unmask.
   Then, the predictor determines the symbols at those positions by simulating the behavior of the \pCoTAcr transformer.
   This is trivial if the \eMDMAcr is causally masked like the \pCoTAcr transformer.
   However, if the \eMDMAcr is not causally masked, the planner must take additional steps to ensure that the prediction is equivalent to the \pCoTAcr transformer.
   
   Concretely, the \eMDMAcr simulates the \pCoTAcr transformer on the input string $\str$ of length $\strlen$ as follows:
   \begin{enumerate}
      \item We first note that the \pCoTAcr transformer predicts the next $\padlen'$ symbols at every timestep $t$ based on the input string $\str_{\leq \strlen + (t - 1) \padlen'} \underbrace{\padSym \cdots \padSym}_{(\timesteps - t) \padlen'}$ rather than $\str_{\leq \strlen + (t - 1) \padlen'}$. 
      This is because, by \cref{lem:ignore-positions}, the \pCoTAcr transformer can ignore all symbols containing the padding symbol and thus produce equivalent predictions at every step $t$.
      This will help us make use of the same padding space at every step of the simulation.
      \item We assume that the final output of the \eMDMAcr will be stored in the first $\padlen$ positions of the padding space.
      It will be filled in $\timesteps$ generation steps, where at each step $t \in \NTo{\timesteps}$, a new block of $\padlen'$ symbols will be predicted.
      In particular, by \cref{lem:select-block}, the planner can select the next $\padlen'$ positions to unmask at time $t$ by including the values $\sbinFun{\ceil{\sfrac{\posPart{\stridx - \strlen}}{\padlen'}}}$ and $\sbinFun{\ceil{\sfrac{\posPart{\stridx - \strlen}}{\padlen'}} - 1}$ in the positional encodings.
      \item The predictor then uses an initial layer to copy the input string $\str_{\leq \strlen + (t - 1) \padlen'} \underbrace{\padSym \cdots \padSym}_{(\timesteps - t) \padlen'}$ into the residual stream of the first $\strlen + \padlen$ positions of the padding space.
      \item The predictor then uses the $(\strlen + \padlen)^2$ padding positions to predict the next $\padlen'$ symbols predicted by the \pCoTAcr transformer.
      These are written to the positions chosen to be unmasked by the planner.
   \end{enumerate}
\end{proof}

\MDMsCoTChain*
\begin{proof} 
   For simplicity, we assume that the input to the \pCoTAcr transformer is padded by $\padlen$ symbols.
   Intuitively, the \pCoTAcr transformer simulates an $\numlayers$-layer \eMDMAcr transformer on the input string $\str$ of length $\strlen$ by simulating each \eMDMAcr generation step with additional padding to account for causal masking.
   Whereas the \eMDMAcr transformer ``overwrites'' its previous input and bases its predictions at time step $t$ on the current version of the unmasked input, the \pCoTAcr transformer bases its predictions of the $\padlen$ symbols on the entire string of $(t - 1) \cdot \padlen$ symbols generated so far.
   For correct simulation, the \pCoTAcr transformer therefore has to ignore all the symbols not generated at the previous time step, which will be ensured by appropriate positional encodings. 
   The \pCoTAcr transformer can then predict the next $\strlen + \padlen$ symbols based on the current input string and the previously predicted symbols, simulating the behavior of the \eMDMAcr transformer on that input.
   However, to predict $\strlen + \padlen$ symbols, the \pCoTAcr transformer uses $\padlen' \defeq \numlayers (\strlen + \padlen)$ padding space at each step to account for the unmasked nature of the \eMDMAcr transformer (cf. \cref{lem:unmask-conversion-fp}).
   
   More concretely, the simulation happens as follows.
   \begin{enumerate}[label=(\arabic*)]
      \item The \pCoTAcr transformer uses additional positional encodings with the information about $\floor{\sfrac{\posPart{\stridx - \strlen}}{\timesteps}}$, $\posPart{\stridx - \strlen} \mod \padlen'$, $\floor{\sfrac{\posPart{\stridx - \strlen}}{\numlayers}}$, and $\posPart{\stridx - \strlen} \mod \padlen$.
      These positional encodings allow the \pCoTAcr transformer to identify 
      \begin{enumerate*}[label=\textit{(\arabic*)}]
         \item the previous block of $\padlen'$ predicted symbols,
         \item the last $\padlen$ symbols within that block (which is where the actual predictions of the previous step will be stored), and
         \item the current position in the block
      \end{enumerate*}
      with \cref{lem:focus-on-marked-positions}.
      \item The \pCoTAcr transformer first uses an initial layer to copy the output of the previous generation step (captured in the previous $\strlen + \padlen$ positions) into the next $\strlen + \padlen$ positions of the padding space (this is where we use the assumption that the input to the \pCoTAcr transformer is padded---if that is not the case, a more complicated construction could specifically handle the initial step of the generation where only the initial input string would be copied).
      \item The \pCoTAcr transformer can then predict the next $\padlen$ symbols by simulating the behavior of the composed \eMDMAcr planner and predictor as in \cref{thm:lts-can-simulate-mdms}.
   \end{enumerate}
\end{proof}

\end{document}